 \documentclass[10pt, final, journal, letterpaper, twocolumn]{IEEEtran}
\ifCLASSINFOpdf
\else
\fi
\usepackage{amsmath,amssymb}
 \usepackage{subfigure}
 \usepackage{graphicx,graphics,color,psfrag}
 \usepackage{cite,balance}
 \usepackage{caption}
 \allowdisplaybreaks
 \usepackage{algorithm}
 \usepackage{accents}
 \usepackage{amsthm}
 \usepackage{amsmath}
 \usepackage{bm}
 \usepackage{algorithmic}
 \usepackage[english]{babel}
 \usepackage{multirow}
 \usepackage{enumerate}
 \usepackage{cases}
 \usepackage{stfloats}
 \usepackage{dsfont}
 \usepackage{color,soul}
 \usepackage{amsfonts}
 \usepackage{cite,graphicx,amsmath,amssymb}
 \usepackage{subfigure}
 \usepackage{fancyhdr}
 \usepackage{hhline}
 \usepackage{graphicx,graphics}
 \usepackage{array,color}
 \usepackage{amsmath}
 \usepackage{booktabs}
 \usepackage{ulem}  

\makeatletter
\renewcommand{\fnum@figure}{Fig. \thefigure}
\makeatother

\newtheoremstyle{itshape}
  {.0\baselineskip\@plus.0\baselineskip\@minus.0\baselineskip}
  {.0\baselineskip\@plus.0\baselineskip\@minus.0\baselineskip}
  {\itshape}
  {}
  {\bfseries}
  {.}
  { }
  {}

\theoremstyle{itshape}

\newtheorem{theorem}{Theorem}

\newtheorem{lemma}{Lemma}

\newtheorem{assumption}{Assumption}
\newtheorem{definition}{Definition}

\begin{document}
\include{header}
\title{Asynchronous Wireless Federated Learning with Probabilistic Client Selection}

\author{Jiarong Yang, Yuan Liu, Fangjiong Chen, Wen Chen, and Changle Li

\thanks{
This paper was supported in part by the National Key Research and Development Program of China under Grant 2020YFB1807700, in part by the National Natural Science Foundation of China under Grant 61971196, U2001210, and 62071296, and in part by Shanghai 22JC1404000, 20JC1416502, and PKX2021-D02. \emph{Corresponding authors: Yuan Liu  and Fangjiong Chen.} 
\par
Jiarong Yang, Yuan Liu and Fangjiong Chen are with school of Electronic and Information Engineering, South China University of Technology, Guangzhou 510641, China (e-mails: eejryang@mail.scut.edu.cn, eeyliu@scut.edu.cn, eefjchen@scut.edu.cn). 
\par
Wen Chen is with Department of Electronic Engineering, Shanghai Jiao Tong University, Shanghai 200240, China (e-mail: wenchen@sjtu.edu.cn).
\par
Changle Li is with the State Key Laboratory of Integrated Services Networks and Research Institute of Smart Transportation, Xidian University, Xi’an 710071, Shaanxi, China (e-mail: clli@mail.xidian.edu.cn).
}
}

\maketitle

\vspace{-1.5cm}
\begin{abstract}
Federated learning (FL) is a promising distributed learning framework where distributed clients collaboratively train a machine learning model coordinated by a server. To tackle the stragglers issue in asynchronous FL, we consider that each client keeps local updates and probabilistically transmits the local model to the server at arbitrary times. We first derive the (approximate) expression for the convergence rate based on the probabilistic client selection. Then, an optimization problem is formulated to trade off the convergence rate of asynchronous FL and mobile energy consumption by joint probabilistic client selection and bandwidth allocation. We develop an iterative algorithm to solve the non-convex problem globally optimally. Experiments demonstrate the superiority of the proposed approach compared with the traditional schemes.
\end{abstract}


\section{Introduction}
\label{Introduction}
The burgeoning of the Internet of Things (IoT) and the fifth generation mobile communication (5G) technology are driving the era of the Internet of Everything (IoE). Smartphones, wearable devices, smart homes, and various distributed wireless sensors generate massive amounts of data in real time \cite{shi2017edge, 9786891, 8166725}, which can be used by machine learning techniques to bring diverse artificial intelligence (AI) services to people \cite{lecun2015deep, 8770530, 8970161}. However, traditional machine learning is based on centralized training that requires sending all raw data to servers \cite{9098940, 9847105}, which can drain the wireless bandwidth and cause long communication delays. In addition, with the increasing public concern about data security and privacy issues, people or companies are not willing to share their private data, which leads to the problem of data islands \cite{chen2012data}. To address the aforementioned issues, a promising distributed machine learning framework called federated learning (FL) has been proposed \cite{mcmahan2017communication, 10.1145/3298981}. A typical FL consists of a server and multiple clients with computation capabilities, in which the server only aggregates local models trained at the client side to obtain the global model. By decentralizing model training to the client side, FL not only preserves the data privacy, but also uses the computation capabilities of clients \cite{lim2020federated}. 
\par
FL always assumes synchronous client-server communication, where the server waits to receive all local models of clients for model aggregation. However, with limited bandwidth (computation) resources, some clients may suffer from high communication (computation) latency and thus slow down convergence rate of the whole system \cite{avdiukhin2021federated}. This is known as \emph{stragglers}. Asynchronous FL can naturally solve the stragglers issue, where the server aggregates local models without waiting for the delayed clients \cite{xie2019asynchronous, chen2019communication, anonymous2020efficient, chen2020asynchronous, avdiukhin2021federated, lee2021adaptive, ma2021fedsa, zhu2022online, wang2022asynchronous, you2022triple}. However, discarding stragglers in asynchronous FL \cite{you2022triple, wang2022asynchronous} will cause serious performance degradation. Specifically, on the one hand, if stragglers do not receive the global model broadcast by the server due to network failure (e.g., channel fading), they cannot perform local training and communicate with the server. Thus stragglers not only violate participation fairness but also may drift the global model towards local optimum \cite{huang2022stochastic}. On the other hand, the gradients of stragglers based on the stale global model may also degrade the model performance \cite{chen2020asynchronous, you2022triple}. To solve this problem, the authors in \cite{avdiukhin2021federated} propose a scheme that clients can keep local model training based on the last received global model and send the updated local model to the server at arbitrary time.  In this paper, we consider this asynchronous FL scenario in wireless networks.
\par
As the synchronous case, the asynchronous FL also faces energy-efficiency problem in wireless networks, as mobile clients are usually with limited battery capacity while the transmission of local models is significant energy consuming. In addition, over-participation of some clients may hurt the model performance, i.e., too much participation of a certain group of clients may “drift” the global model to local optimum \cite{bian2022mobility, huang2022stochastic, cho2020client}. On the other hand, increasing the client participants could improve the model convergence \cite{mcmahan2017communication}. This problem is known as the \emph{client selection} that is used as an important method to deal with energy efficiency \cite{xu2020client, sun2021dynamic} and data heterogeneity \cite{yang2022client}, i.e., selecting the clients with the most informativeness and/or the least communication overhead to participate in training. Note that client selection is more difficult in considered asynchronous FL with continuous local training and arbitrary communication. Specifically, in traditional (both synchronous and asynchronous) FL, client selection is performed regularly at the beginning of each round, and clients are selected to participate in training in a deterministic way, i.e., the server gathers the network information to calculate the selection or scheduling policy and then selects clients centrally. Whereas, in considered asynchronous FL,  where clients can send their updated local models to server at arbitrary times, every client can be selected by a probability at each time. In other words, the server only calculates the selection probability for clients, and each client autonomously decides whether to send updated local model to server based on its selection probability. This refers to as \emph{probabilistic client selection} in this paper. However, how to use probabilistic client selection to handle energy efficiency and data heterogeneity has no attention. 
\subsection{Contributions}
In this paper, we study the stragglers issue of asynchronous FL by joint probabilistic client selection and bandwidth allocation, where the clients keep local updates and exchange models with the server at arbitrary times. The main contributions of this paper are summarized as follows:
\begin{itemize}
\item We analyze the convergence rate of the asynchronous FL with probabilistic client selection, which exhibits the relationship between the training performance and participation probability of the clients. In particular, we consider an individual interval that each client must communicate with the server at least once, so as to adapt channel diversity and participation fairness.
\item We formulate a joint optimization problem of probabilistic client selection and bandwidth allocation to trade off the convergence performance of the asynchronous FL and energy consumption of the clients. We use the sum-of-ratios algorithm to transform the original non-convex problem into a convex subtractive form and develop an iterative algorithm to find the globally optimal solutions.
\item Several useful insights are obtained via our analysis: First, more frequent client-server communication can accelerate the model convergence. Second, fair participation of clients can also improve the model performance. Third, more clients benefit model performance, but if the number of clients exceed a certain threshold, the model performance degrades due to data heterogeneity. Fourth, the proposed scheme not only enhances both model performance and energy efficiency, but also improves fairness of energy consumption among clients.
\end{itemize}
\subsection{Prior Work}
Wireless resource allocation and client selection for synchronous FL has been extensively studied \cite{shi2020joint, xu2020client, sun2021dynamic, wang2019adaptive, chen2020joint, amiri2021convergence, yang2020age, zhu2020one, 9760232, 9685090}, with the goal of achieving higher model accuracy or lower energy/latency cost with limited resources. For example, the authors in \cite{shi2020joint} consider joint client selection and bandwidth allocation to maximize the model accuracy obtained in a limited training time. The authors in \cite{xu2020client} show that selecting more clients at the end of training phase is beneficial for model accuracy improvement, and then propose an energy-aware client selection and bandwidth allocation algorithm. The authors in \cite{sun2021dynamic} consider an over-the-air FL with dynamic client selection to minimize the convergence bound while meeting energy constraints. In \cite{wang2019adaptive}, the maximum model accuracy is achieved by weighting local updates for global parameter aggregation under a given resource budget. In \cite{yang2020age}, an age-based selection policy is proposed to jointly consider the staleness of the received parameters and the instantaneous channel qualities. Furthermore, client selection based on data importance has been employed to optimize the performance of FL in wireless networks \cite{9107235, 9170917, 9252927}. Specifically, the importance metric of data in \cite{9107235} incorporates both the signal-to-noise ratio (SNR) and data uncertainty. \cite{9170917} utilizes gradient divergence to quantify the importance of each local update. And in \cite{9252927}, the importance of data is quantified by calculating the inner product between each local gradient and the ground-truth global gradient. In addition, client selection with the consideration of fairness is studied in \cite{huang2022stochastic, li2021federated}, where the probability of selecting each client is ensured by introducing a fairness constraint.
\par

Asynchronous FL has been extensively studied to address the stragglers issue. In \cite{xie2019asynchronous} and \cite{chen2019communication} the local updates uploaded by clients are adaptively weighted based on their staleness, aiming to alleviate the impact of stragglers on model performance. In \cite{anonymous2020efficient}, the number of local epochs is dynamically adjusted with the estimation of staleness to improve the model convergence. In \cite{chen2020asynchronous}, a dynamic learning strategy for step size adaptation on local training is proposed to reduce the impact of staleness caused by stragglers. In \cite{avdiukhin2021federated}, an asynchronous version of local stochastic gradient descent (SGD) is studied, wherein the clients can communicate with the server at arbitrary time intervals.
The above asynchronous FL works do not consider wireless networks.
\par
There are some researches on asynchronous FL over wireless networks \cite{lee2021adaptive, ma2021fedsa, wang2022asynchronous, zhu2022online}.
For example, The authors in \cite{lee2021adaptive} propose three client selection strategies that maximize the expected number of effective training samples. In \cite{ma2021fedsa}, a semi-asynchronous FL is proposed to minimize the training completion time, where the server aggregates local models based on their arrival order.
In \cite{zhu2022online}, an adaptive client selection algorithm is proposed to minimize training latency while taking into account the client availability and long-term fairness.
The staleness of local models and the delays of local model training and uploading are considered in \cite{wang2022asynchronous}.  
However, the works \cite{lee2021adaptive, ma2021fedsa, zhu2022online} assume that the fixed number of clients are scheduled for aggregation. And in \cite{wang2022asynchronous}, stragglers are discarded if they cannot transmit within a period. Unlike \cite{lee2021adaptive, ma2021fedsa, wang2022asynchronous, zhu2022online}, in our paper, clients probabilistically participate in the training and independently decide whether to upload their local updates.
\par
It is necessary to compare our paper with the work \cite{avdiukhin2021federated} which considers the same asynchronous FL scenario of continuous local updates and arbitrary communication. First, the work \cite{avdiukhin2021federated} focuses on convergence analysis based on the assumption that all clients communicate with the server at the same frequency, while our paper gets away from this impractical assumption and obtains the new results. Second, our paper focuses on client selection and bandwidth allocation, which are not considered in \cite{avdiukhin2021federated}. To sum up, our paper and \cite{avdiukhin2021federated} consider the same asynchronous FL scenario but different assumptions and focuses, thus the two works are fundamentally different.
\par
The rest of this paper is organized as follows. In Section \ref{sec:system_model_and_problem_formulation}, we present the system model. We analyze the convergence rate and formulate the optimization problem in Section \ref{Convergence_Analysis_and_Problem_Formulation}. An optimal algorithm to solve the formulated problem is proposed in Section \ref{sec:Joint_Sending_Probability_Optimization_and_Bandwidth_Allocation}. Finally, the experiment results and conclusions are provided in Section \ref{sec:Experimental_Results} and Section \ref{sec:Conclusion}, respectively.

\section {System Model}
\label{sec:system_model_and_problem_formulation}
In this section, we describe the system model for asynchronous FL over wireless networks. The main notations used in this paper are summarized in Table \ref{table0}.

\begin{table}[]
\centering
\caption{Key notations.}
\label{table0}
\begin{tabular}{cl}
\hline
\multicolumn{1}{l}{\textbf{Symbols}}                             & \multicolumn{1}{l}{\textbf{Definitions}}                                                       \\ \hline
$\boldsymbol{x}_t$                                               & Global model parameters at round $t$                                \\
$\boldsymbol{x}_{k,t}$                                           & Model parameters of client $k$ at round $t$                         \\
$\boldsymbol{y}_{k,t}$                                           & Global model parameters received by the client $k$ before \\
\multicolumn{1}{l}{}                                             & round $t$                                                           \\
$\boldsymbol{\delta}_{k,t}$                                      & Actual gradient descent direction of client $k$ at round $t$        \\
$f_t$                                                            & Global loss function at round $t$                                   \\
$f_{k,t}$                                                        & Local loss function of client $k$ at round $t$                      \\
$C_t$                                                            & Set of clients communicating with the server at round $t$           \\
$\Delta_k$                                                       & Maximum communication interval between client k and                 \\
\multicolumn{1}{l}{}                                             & the server                                                          \\
$w_{k,t}$                                                        & Bandwidth allocation ratio for the channel of client $k$ at         \\
\multicolumn{1}{l}{}                                             & round $t$                                                           \\
$P_k$                                                            & Transmit power of client $k$                                        \\
$p_{k,t}$                                                        & Probability of client $k$ communicating with the server at          \\
\multicolumn{1}{l}{}                                             & round $t$                                                           \\
$E_t$                                                            & Total expected energy consumption (Joule) at round $t$              \\
$\nabla F_{k,t}$                                                 & Unbiased stochastic gradient of $f_k$ at $\boldsymbol{x}_{k,t}$     \\
$\mathbf{G}_{max}$                                               & Maximum stochastic gradient norm                                    \\
$\sigma^2$                                                       & Variance of stochastic gradients                                    \\
$\boldsymbol{x}^*$                                               & Optimal global model parameters                                     \\
$f_{max}$                                                        & Difference between the initial global loss function and the         \\
\multicolumn{1}{l}{}                                             & optimal global loss function                                        \\
$L$                                                              & Smoothness parameter of $f_k$                                       \\
$\lambda$                                                        & Minimum client selection probability                                \\
$\rho$                                                           & Tradeoff coefficient for bi-objective optimization problem          \\
$(\boldsymbol{\alpha}, \boldsymbol{\beta}, \boldsymbol{\gamma})$ & Auxiliary variables introduced in fractional programming            \\ 
\multicolumn{1}{l}{}                                             & transform     \\\hline                                                     
\end{tabular}
\end{table}

\subsection{Asynchronous FL Model}

We consider an asynchronous FL system comprising a server and $K$ clients indexed by $\mathcal{K}=\{1,2,\cdots,K\}$. According to the FL framework, a global statistical model is trained in a distributed manner to minimize a global loss function as follows:
\begin{align}
    \min\limits_{\boldsymbol{x}} f(\boldsymbol{x})=\frac1K\sum_{k=1}^K f_k(\boldsymbol{x}),
\end{align}
where $\boldsymbol{x}$ denotes the global model parameters and $f_k$ is the local loss function based on the data set on client $k$.
We consider an asynchronous FL scenario, in which all clients update the local models at all time but send the latest gradients to the server at arbitrary times. 
Specifically, for client $k$ at round $t$, we use $\boldsymbol{x}_{k,t}$ to denote the local model parameters and $\boldsymbol{y}_{k,t}$ to denote the global model parameters received by client $k$ before round $t$. 
Then the actual gradient descent direction of client $k$ at round $t$ is defined as a pseudo-gradient \cite{charles2021large, reddi2020adaptive} as follows
\begin{align}\label{eqn:local_update}
\boldsymbol{\delta}_{k,t} = \boldsymbol{x}_{k,t} - \boldsymbol{y}_{k, t-1}.
\end{align}
\par
The server receives the gradients sent by the clients and updates the global parameters as follows:
\begin{align}\label{eqn:global_update}
\boldsymbol{x}_t = \boldsymbol{x}_{t-1} + \frac{1}{K}\sum_{k\in C_t}\boldsymbol{\delta}_{k,t},
\end{align}
where $C_t$ denotes the set of clients who communicate with (or send gradients to) the server at round $t$. After updating the global model, the sever broadcasts the updated parameters to the clients in $C_t$.
\par
Note that in traditional synchronous FL, only the selected clients perform local training at each round, whereas in considered asynchronous FL, all clients keep local training and send their local pseudo-gradients to the server at arbitrary times. To overcome the stragglers issue as mentioned in Section \ref{Introduction}, we restrict that each client $k$ must communicate with the server at least once within period $\Delta_k$. By setting a reasonable $\Delta_k$ for each client, it ensures that each client can transmit its local pseudo-gradient to the server at the appropriate time.

\subsection{Communication Model}
We consider orthogonal uplink channel access with total bandwidth $W$, where the bandwidth allocated to client $k$ at round $t$ is $w_{k,t}W$ and $w_{k,t}$ is the bandwidth allocation ratio, satisfying $0\leq w_{k,t}\leq 1$ and $\sum_{k=1}^Kw_{k,t}\leq 1$ for each round $t$. Denote $P_k$ as the transmit power of client $k$, $h_{k,t}$ as the uplink channel gain at round $t$, and $N_0$ as the noise power density. Thus, the achievable transmission rate (bits/s) of client $k$ at round $t$ is given as
\begin{align}\label{eqn:trans_rate}
R_{k,t} = w_{k,t} W\log\left(1+\frac{P_k h_{k,t}}{w_{k,t} W N_0}\right).
\end{align}
\par
Define $p_{k,t}$ as the probability of client $k$ communicating with the server at round $t$, we can obtain the total expected energy consumption (Joule) at round $t$ as 
\begin{align}\label{eqn:ener_total_comsumption}
E_t = \sum_{k=1}^K \frac{p_{k,t} P_k S}{R_{k,t}},
\end{align}
where $S$ is the size (in bits) of model parameters that are the same for all clients. Note that here we only consider the communication/transmission energy, but the energy consumption of local model training, i.e., computing energy,  can be easily involved since it is a constant during the training process.

\subsection{Protocol of Asynchronous FL}
In asynchronous FL networks, the server optimizes the client selection probability $\{p_{k,t}\}_{k=1}^K$ and the bandwidth allocation ratio $\{w_{k,t}\}_{k=1}^K$. 
At each round, the server broadcasts the client selection probability $\{p_{k,t}\}_{k=1}^K$ to all clients. Clients then decide whether to send their gradients based on $\{p_{k,t}\}_{k=1}^K$, where clients who decide to communicate with the server transmit local pseudo-gradients on allocated channels with the bandwidth allocation ratio $\{w_{k,t}\}_{k=1}^K$.
As illustrated in Fig. \ref{fig:system}, the specific steps of protocol can be summarized as follows:
\begin{itemize}
\item{\bf Step 1:} All clients compute the local pseudo-gradients as \eqref{eqn:local_update}. 
\item{\bf Step 2:} The server calculates the client selection probability $\{p_{k,t}\}_{k=1}^K$ and the bandwidth allocation ratio $\{w_{k,t}\}_{k=1}^K$, and broadcasts the client selection probability $\{p_{k,t}\}_{k=1}^K$ to all clients.
\item{\bf Step 3:} The clients decide whether to send local pseudo-gradients based on the client selection probability $\{p_{k,t}\}_{k=1}^K$.
\item{\bf Step 4:} The client $k\in C_t$ sends the local pseudo-gradients to the server via a sub-channel with the bandwidth allocation ratio $\{w_{k,t}\}_{k=1}^K$.
\item{\bf Step 5:} The server updates the global model parameters as \eqref{eqn:global_update} and broadcasts the updated model to the clients $k\in C_t$.
\end{itemize}

\begin{figure}[t]
\begin{centering}
\includegraphics[width=8.8cm]{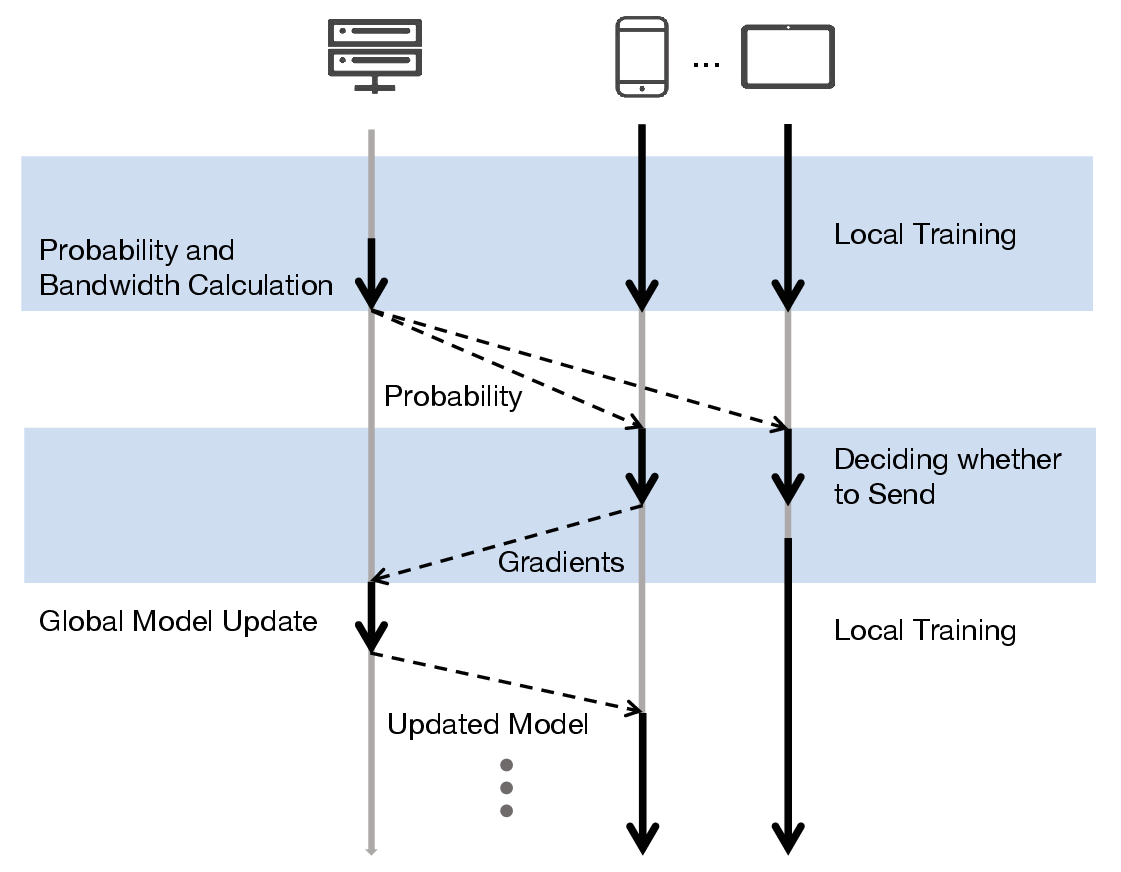}
\caption{Communication protocol for asynchronous FL.}\label{fig:system}
\end{centering}
\end{figure}

\section{Convergence Analysis and Problem Formulation}
\label{Convergence_Analysis_and_Problem_Formulation}
In this section, we first analyze the convergence rate of the proposed asynchronous FL scheme and derive an approximate expression with respect to the client selection probability. Then a joint optimization problem of probabilistic client selection and bandwidth allocation is presented to trade off convergence rate and energy consumption for the asynchronous FL. 
\subsection{Convergence Analysis}
In this subsection, we analyze the convergence rate of the proposed asynchronous FL scheme, which is the basis of the problem formulation. Firstly, we introduce the maximum communication interval $\Delta_k$ between each client $k$ and the server and then derive the convergence rate. Then we approximate the $\Delta_k$ in terms of the client selection probability $\{p_{k,t}\}_{t=0}^{T-1}$ and reveal the relationship between $\{p_{k,t}\}_{t=0}^{T-1}$ and the convergence rate of the model.
\par
Before the convergence analysis, we first introduce the following assumptions.
\begin{assumption}\label{assumption1}
Denote $\nabla F_{k,t}$ as a gradient computed by client $k$ using SGD at round $t$ and it is also an unbiased stochastic gradient of $f_k$ at $\boldsymbol{x}_{k,t}$, satisfying $\mathbb{E}[\nabla F_{k,t}] = \nabla f_k(\boldsymbol{x}_{k,t})$. We assume that
\begin{itemize}
\item Each local loss function $f_k(\boldsymbol{x})$ is L-smooth, i.e., $\|f_k(\boldsymbol{x})-f_k(\boldsymbol{x}')\| \leq L\|\boldsymbol{x}-\boldsymbol{x}'\|$, for any $\boldsymbol{x}$ and $\boldsymbol{x}'$.
\item For any $k$ and $\boldsymbol{x}$, the norm of $\nabla F_{k,t}$ can be bounded by $\mathbb{E}[\|\nabla F_{k,t}\|]^2 \leq \mathbf{G}_{max}^2$.
\item For any $k$, the difference between $\nabla F_{k,t}$ and $\nabla f_k(\boldsymbol{x}_{k,t})$ can be bounded by $\mathbb{E}[\|\nabla F_{k,t}-\nabla f_k(\boldsymbol{x}_{k,t})\|^2] \leq \sigma^2 $.
\end{itemize}
Here $\mathbb{E}[\cdot]$ is the expectation of the randomness results from stochastic gradients. The first assumption is widely used in in the literature of convergence analysis for FL \cite{shi2020joint,wang2019adaptive,yang2019scheduling}, while the second and the third assumptions are standard assumptions in the SGD literature \cite{stich2018local,yu2019parallel}.
\end{assumption}
\par
By assigning each client an individual maximum communication interval $\Delta_k$, the convergence rate of asynchronous FL is shown in the following Lemma.
\begin{lemma}\label{lemma1}
When Assumptions \ref{assumption1} holds, given the maximum communication interval $\Delta_k, k \in \{1,2,\cdots,K\}$ and $\eta \leq \frac{1}{8L}$, the convergence rate of asynchronous FL is given by
\begin{align}
\frac{1}{T}\sum_{t=0}^{T-1}\mathbb{E}&[\|\nabla f(\boldsymbol{x}_t)\|^2]  \nonumber\\
&\leq \frac{8f_{max}}{\eta T} + 92\eta^2 L^2 \mathbf{G}_{max}^2 (\frac{\sum_{k=1}^K\Delta_k^2}{K}) + 9\sigma^2,
\end{align}
where $L$, $\mathbf{G}_{max}$, and $\sigma$ are constants in Assumption \ref{assumption1}, $\eta$ is the learning rate, $f_{max} = f(\boldsymbol{x}_0)-f(\boldsymbol{x}^*)$ and $\boldsymbol{x}^*$ is the optimal global model parameters.
\end{lemma}
\begin{proof}
See Appendix A.
\end{proof}
Lemma \ref{lemma1} reveals that the smaller $\Delta_k$ of clients, the better the convergence rate of the model, confirming that the more frequently the clients exchange models with the server, the better the performance of the model. 
It is worth noting that, compared with \cite{avdiukhin2021federated} assigning a common maximum communication interval $\Delta$ for all clients, it is non-trivial that we assign individual $\Delta_k$ for each client $k$ due to two-fold: First, the derivation in this paper is significantly different and thus we obtain a new convergence rate, which also affects the next derivations of selection probability and considered optimization problem. Second, a common $\Delta$ makes all clients send gradients at the same frequency, which ignores the heterogeneous channel conditions of clients. On the contrary, the proposed individual $\Delta_k$ can adapt channel variations and multi-user diversity, i.e., allowing the clients with better channel conditions to send gradients more frequently and those with poor channel conditions to send less frequently.
\par
Furthermore, the client selection probability $\{p_{k,t}\}_{t=0}^{T-1}$ is directly related to the maximum communication interval $\Delta_k$ as follows.
\begin{align}\label{eqn:delta_pik}
\mathbb{E}[\Delta_k] = \sum_{t=0}^{T-1}\left[\left(p_{k,t}\prod_{\tau=0}^{t-1}(1-p_{k,\tau})\right)t\right].
\end{align}
\eqref{eqn:delta_pik} is intractable as variables $p_{k,t}$ are multiplied, so we approximate $\Delta_k$ to a tractable form by the following way. First, based on $\{p_{k,t}\}_{t=0}^{T-1}$, we can obtain the expected number of rounds in which client $k$ communicates with the server in total $T$ rounds.
Subsequently, we assume that all clients exchange models with the server \emph{periodically} throughout the $T$ rounds, so the approximate $\Delta_k$ can be given by
\begin{align}\label{eqn:delta_pik_app}
\Delta_k' = \frac{T}{\sum_{t=0}^{T-1} p_{k,t}}.
\end{align}

Substituting \eqref{eqn:delta_pik_app} into Lemma \ref{lemma1}, we can approximate the upper bound of the convergence rate by the client selection probability $\{p_{k,t}\}_{t=0}^{T-1}$, as shown in the following theorem.
\begin{theorem}\label{theorem1}
Given the client selection probability $p_{k,t}\in\{0,1,\cdots,T-1\}, k\in \{1,2,\cdots,K\}$, the convergence rate of asynchronous FL is approximated by
\begin{align}\label{eqn:theorem1}
\frac{1}{T}&\sum_{t=0}^{T-1}\mathbb{E}[\|\nabla f(\boldsymbol{x}_t)\|^2] 
\leq \frac{8f_{max}}{\eta T} \nonumber \\
&+ 92\eta^2 L^2 \mathbf{G}_{max}^2 \left(\frac{T^2}{K}\sum_{k=1}^K\left(\frac{1}{\sum_{t=0}^{T-1} p_{k,t}}\right)^2\right) + 9\sigma^2.
\end{align}
\end{theorem}
\par
Omitting the constant terms in \eqref{eqn:theorem1}, the convergence rate of the model can be written by
\begin{align}\label{eqn:conv_app}
\frac{1}{T}\sum_{t=0}^{T-1}\mathbb{E}[\|\nabla f(\boldsymbol{x}_t)\|^2] = O\left(\frac{T^2}{K}\sum_{k=1}^K\left(\frac{1}{\sum_{t=0}^{T-1} p_{k,t}}\right)^2\right).
\end{align}
\begin{lemma}
More frequent client-server model exchanges can improve convergence. 
\end{lemma}
\begin{proof}
   The conclusion is apparent since increasing the value of $p_{k,t}$ can decrease the upper bound of the convergence rate of asynchronous FL defined in \eqref{eqn:conv_app}.
\end{proof}
\begin{lemma}\label{lemma3}
Fair participation of clients can improve convergence.
\end{lemma}
\begin{proof}
   See Appendix B.
\end{proof}
When $\Delta_k'=\Delta$ for all $k$ in Lemma \ref{lemma3}, i.e., each client exchanges model with the server with the same period, it is actually the age-based scheme \cite{yang2020age} and the optimal solution is Round-Robin scheduling, i.e., each client participates in training by a sequential and alternating manner.

\par
\subsection{Problem Formulation}
In this paper, a joint optimization problem of probabilistic client selection and bandwidth allocation is formulated to trade off the convergence rate and total energy consumption of the asynchronous FL. Specifically, we use \eqref{eqn:conv_app} as the metric of the convergence rate of asynchronous FL.
Furthermore, the energy consumption for a single round $ E_t$ is obtained with \eqref{eqn:ener_total_comsumption}, so the total energy consumption within the number of rounds $T$ is defined as $\sum_{t=0}^{T-1}E_t$.
For simplicity, we use $\mathcal{T}$ and $\mathcal{K}$ to denote $\{0,1,\cdots,T-1\}$ and $\{1,2,\cdots,K\}$, respectively.
The optimization problem can be written as:
\begin{align}
(\text{P1}):\quad \min_{\boldsymbol{p}, \boldsymbol{w}} \quad & \frac{\rho T^2}{K}\sum_{k=1}^K\left(\frac{1}{\sum_{t=0}^{T-1} p_{k,t}}\right)^2 \nonumber\\
   +& (1-\rho)\sum_{t=0}^{T-1}\sum_{k=1}^K \frac{p_{k,t} P_k S}{w_{k,t} W\log\left(1+\frac{P_k h_{k,t}}{w_{k,t} W N_0}\right)} \label{eqn:obj}\\
    {\rm{s.t.}} \quad&\sum_{k=1}^{K} w_{k,t} \leq 1 , \quad\quad~\forall t\in\mathcal{T}, \label{eqn:p1_2}\\
    &0 \leq w_{k,t} \leq 1,  \quad\quad~\forall t\in\mathcal{T},\forall k\in\mathcal{K}, \label{eqn:p1_3}\\
    &\lambda \leq p_{k,t} \leq 1,  \quad\quad~\forall t\in\mathcal{T},\forall k\in\mathcal{K}, \label{eqn:p1_4}
\end{align}
where $\boldsymbol{p} = \{\boldsymbol{p}_1,\boldsymbol{p}_2,\cdots,\boldsymbol{p}_K\}$ with $\boldsymbol{p}_k = \{p_{k,0}, p_{k,1},\cdots,p_{k,T-1}\}$ denotes the client selection probability vector, $\boldsymbol{w} = \{\boldsymbol{w}_1,\boldsymbol{w}_2,\cdots,\boldsymbol{w}_K\}$ with $\boldsymbol{w}_k = \{w_{k,0}, w_{k,1},\cdots,w_{k,T-1}\}$ denotes the bandwidth allocation ratio vector, $\lambda$ is the minimum client selection probability and $\rho$ is the tradeoff coefficient for bi-objective optimization problem, tending to optimize the convergence rate of asynchronous FL when $\rho$ is large and to minimize the total energy consumption when $\rho$ is small.
Constraints \eqref{eqn:p1_2} and \eqref{eqn:p1_3} are the feasibility conditions on the bandwidth allocation, and constraint \eqref{eqn:p1_4} ensures that each client is given at least a certain chance to participate in the asynchronous FL.
\par
However, problem (P1) is a non-convex problem as the second term of objective is in the form of sum-of-ratios. So problem (P1) needs to be transformed into equivalently tractable problems by using suitable mathematical tools.

\section{Optimal Solution for Problem (P1)}
\label{sec:Joint_Sending_Probability_Optimization_and_Bandwidth_Allocation}
In this section, we convert problem (P1) into an equivalent parameterized subtractive-form problem and develop iterative algorithms to find the globally optimal solutions.
\subsection{Fractional Programming Transform}
Problem (P1) is a non-convex problem as the second term of \eqref{eqn:obj} is the sum-of-ratios structure. According to \cite{jong2012efficient}, we transform the sum-ratio optimization problem into a parameterized subtractive-form problem as shown in the following theorem.
\begin{theorem}\label{theorem2}
Denote $\boldsymbol{\alpha} = \{\boldsymbol{\alpha}_1,\boldsymbol{\alpha}_2,\cdots,\boldsymbol{\alpha}_K\}$ with $\boldsymbol{\alpha}_k = \{\alpha_{k,0}, \alpha_{k,1},\cdots,\alpha_{k,T-1}\}$, $\boldsymbol{\beta} = \{\boldsymbol{\beta}_1,\boldsymbol{\beta}_2,\cdots,\boldsymbol{\beta}_K\}$ with $\boldsymbol{\beta}_k = \{\beta_{k,0}, \beta_{k,1},\cdots,\beta_{k,T-1}\}$ and $\boldsymbol{\gamma} = \{\gamma_1, \gamma_2,\cdots,\gamma_K\}$, if $(\boldsymbol{p}^*, \boldsymbol{w}^*)$ is the optimal solution to problem (P1), there exists $\boldsymbol{\alpha}$, $\boldsymbol{\beta}$ and $\boldsymbol{\gamma}$ such that $(\boldsymbol{p}^*, \boldsymbol{w}^*)$ is also the optimal solution to the following problem with $\boldsymbol{\alpha}=\boldsymbol{\alpha}^*$, $\boldsymbol{\beta} = \boldsymbol{\beta}^*$ and $\boldsymbol{\gamma} = \boldsymbol{\gamma}^*$:
\begin{align}
    (\text{P2}):\quad \min_{\boldsymbol{p}, \boldsymbol{w}} \quad & \frac{\rho T^2}{K}\sum_{k=1}^K\left(\frac{1}{(\sum_{t=0}^{T-1} p_{k,t})^2}-\gamma_k\right) \nonumber\\
    &~~+\sum_{t=0}^{T-1}\sum_{k=1}^K\alpha_{k,t}\bigg(p_{k,t}P_kS(1-\rho) \nonumber \\
    &~~-\beta_{k,t}w_{k,t} W\log\left(1+\frac{P_k h_{k,t}}{w_{k,t} W N_0}\right)\bigg)\\
    {\rm{s.t.}} \quad&\sum_{k=1}^{K} w_{k,t} \leq 1 , \quad\quad~\forall t\in\mathcal{T}, \\
    &0 \leq w_{k,t} \leq 1,  \quad\quad~\forall t\in\mathcal{T},\forall k\in\mathcal{K}, \\
    &\lambda \leq p_{k,t} \leq 1,  \quad\quad~\forall t\in\mathcal{T},\forall k\in\mathcal{K}. 
\end{align}
Moreover, $(\boldsymbol{p}^*, \boldsymbol{w}^*)$ also satisfies the following equations, when we set $\boldsymbol{\alpha}=\boldsymbol{\alpha}^*$, $\boldsymbol{\beta} = \boldsymbol{\beta}^*$ and $\boldsymbol{\gamma} = \boldsymbol{\gamma}^*$
\begin{align}\label{eqn:cont}
\left\{
\begin{aligned}
&\alpha_{k,t}^*R_{k,t}^*-1=0, ~~\quad\quad\quad\quad\quad\quad~\forall t\in\mathcal{T},\forall k\in\mathcal{K},\\
&\beta_{k,t}^*R_{k,t}^*-p_{k,t}^*P_k S(1-\rho)=0, \quad~\forall t\in\mathcal{T},\forall k\in\mathcal{K},\\
&\gamma_k^* - \frac{\rho T^2}{K(\sum_{t=0}^{T-1} p_{k,t}^*)^2}=0, \quad\quad\quad~\forall k\in\mathcal{K},
\end{aligned}
\right.
\end{align}
where $R_{k,t}^*$ indicates the transmission rate when $w_{k,t}=w_{k,t}^*$.
\end{theorem}
\begin{proof}
See Appendix C.
\end{proof}
\par
 Theorem \ref{theorem2} transforms problem (P1) into an equivalent parameterized form (P2) by introducing additional parameters $(\boldsymbol{\alpha}, \boldsymbol{\beta}, \boldsymbol{\gamma})$, where problems (P1) and (P2) have the same optimal solutions $(\boldsymbol{p}^*, \boldsymbol{w}^*)$ when $(\boldsymbol{\alpha}, \boldsymbol{\beta}, \boldsymbol{\gamma}) = (\boldsymbol{\alpha}^*, \boldsymbol{\beta}^*, \boldsymbol{\gamma}^*)$. Accordingly, problem (P2) can be solved in two layers: in the inner-layer, $(\boldsymbol{p}^*, \boldsymbol{w}^*)$ can be obtained by solving the subtractive-form problem (P2) with given $(\boldsymbol{\alpha}, \boldsymbol{\beta}, \boldsymbol{\gamma})$. Then, in the outer layer, we find the optimal $(\boldsymbol{\alpha}^*, \boldsymbol{\beta}^*, \boldsymbol{\gamma}^*)$ satisfying \eqref{eqn:cont}.

\subsection{Solving Problem (P2) Given $(\boldsymbol{\alpha}, \boldsymbol{\beta}, \boldsymbol{\gamma})$}
Problem (P2) is further decoupled into two sub-problems, namely client selection probability optimization and bandwidth allocation. First, given $(\boldsymbol{\alpha}, \boldsymbol{\gamma})$, the client selection probability optimization problem for each client $k$ can be written as follows.
\begin{align}\label{eqn:p3}
    (\text{P3}):\quad \min_{\boldsymbol{p}} \quad & \frac{\rho T^2}{K}\left(\frac{1}{(\sum_{t=0}^{T-1} p_{k,t})^2}-\gamma_k\right) \nonumber\\
    &+\sum_{t=0}^{T-1}\alpha_{k,t} P_kS(1-\rho)p_{k,t}\\
    {\rm{s.t.}} \quad&\lambda \leq p_{k,t} \leq 1,  \quad\quad~\forall t\in\mathcal{T},\forall k\in\mathcal{K}.  
\end{align}
By solving the $K$ sub-problems (P3), we can obtain the optimal selection probability $\boldsymbol{p^*}$ to problem (P2).
Similarly, given $(\boldsymbol{\alpha}, \boldsymbol{\beta})$, the bandwidth allocation problem for each round $t$ can be written as follows.
\begin{align}
    (\text{P4}):\quad \max_{\boldsymbol{w}} \quad & \sum_{k=1}^K \alpha_{k,t}\beta_{k,t} w_{k,t} W\log\left(1+\frac{P_k h_{k,t}}{w_{k,t} W N_0}\right)\\
    {\rm{s.t.}} \quad&\sum_{k=1}^{K} w_{k,t} \leq 1 , \quad\quad~\forall t\in\mathcal{T}, \\
    &0 \leq w_{k,t} \leq 1,  \quad\quad~\forall t\in\mathcal{T},\forall k\in\mathcal{K}. 
\end{align}
The optimal bandwidth allocation ratio $\boldsymbol{w^*}$ to problem (P2) can also be obtained by solving the $T$ sub-problems (P4).
\par
Now we solve problem (P3) which is a convex problem as all constraints in problem (P3) are affine. We adopt the block coordinate descent (BCD) to solve the problem. Specifically, by denoting \eqref{eqn:p3} as $L_1$, the optimal value of $p_{k,t}$ can be obtained by setting $\frac{\partial L_1}{\partial p_{k,t}} = 0$ as
\begin{align}
\alpha_{k,t}P_kS(1-\rho)-\frac{2\rho T^2}{K(\sum_{t=0}^{T-1} p_{k,t})^3} = 0.
\end{align}
Combining the constraint $\lambda\leq p_{k,t}\leq 1$, $p_{k,t}^*$ can be given by 
\begin{align}\label{eqn:bcd}
    p_{k,t}^* = \left[\left(\frac{2\rho T^2}{K\alpha_{k,t}P_k S(1-\rho)}\right)^{\frac{1}{3}}-\sum_{j\neq t}^{T-1} p_{k,j}\right]_{\lambda}^1.
\end{align}
Thus each $p_{k,t}$ can be updated alternately via \eqref{eqn:bcd} by assuming that other $p_{k,j}$ $(j\neq t)$ are fixed. The alternating optimization method can eventually converge to the optimal solution $\boldsymbol{p}^*$ to problem (P3)  since the problem is convex.
\par
Problem (P4) is also a convex problem with respect to the bandwidth allocation ratio $\boldsymbol{w}$, thus the Lagrangian dual method can be adopted to solve this problem optimally. By introducing the Lagrangian multipliers $\boldsymbol{v} = \{v_0,v_1,\cdots,v_{T-1}\}$ associated with the constraint \eqref{eqn:p1_2}, the Lagrangian function for problem (P4) can be written as 
\begin{align}
    L_2 = \sum_{k=1}^K \alpha_{k,t}\beta_{k,t} w_{k,t} W\log&\left(1+\frac{P_k h_{k,t}}{w_{k,t} W N_0}\right) \nonumber\\
    &+ v_t(1-\sum_{i=1}^K w_{k,t}).
\end{align}
Accordingly, the Lagrangian dual problem is given by
\begin{align}\label{eqn:dual_p}
    \min_{\boldsymbol{v}\succeq0} g_1(\boldsymbol{v}),
\end{align}
where
\begin{align}
    g_1(\boldsymbol{v}) = \max_{0<w_{k,t}\leq1} L_2(\boldsymbol{w}, \boldsymbol{v}).
\end{align}
The second order of the Lagrangian function $L_2$ with respect to $w_{k,t}$ is 
\begin{align}
    \frac{\partial^2L_2}{\partial (w_{k,t})^2} = -\frac{(P_k h_{k,t})^2}{w_{k,t} (W N_0 w_{k,t}+P_k h_{k,t})^2},
\end{align}
which is negative so that $L_2$ is concave over $w_{k,t}$ and the optimal solution can be obtained by setting $\frac{\partial L_2}{\partial w_{k,t}} = 0$. 
Considering the constraint $0\leq w_{k,t} \leq 1$, the closed-form solution of $w_{k,t}^*$ is given by
\begin{align}\label{eqn:optimal_w}
    w_{k,t}^* = \left[\frac{P_k h_{k,t}}{W N_0 \exp[\mathcal{W}(-\exp(-A_{k,t}))+A_{k,t}]-WN_0}\right]_0^1,
\end{align}
where
\begin{align}
    A_{k,t} = \frac{\alpha_{k,t} \beta_{k,t} W+v_t}{\alpha_{k,t} \beta_{k,t} W},
\end{align}
and $\mathcal{W}(x)$ is the principal branch of the Lambert $W$ function which is defined as the inverse function of $f(x)=x\exp(x)$ \cite{corless1996lambertw}. The details of obtaining the close form solution of $w_{k,t}^*$ can be found at Appendix D.
\par
For the dual problem \eqref{eqn:dual_p}, the subgradient method is used to find the optimal Lagrangian multipliers $\boldsymbol{v}$, in which each $v_t$ is updated as
\begin{align}\label{eqn:update_v}
    v_t^{t_1+1} = \left[v_t^{t_1}-\epsilon^{t_1}(1-\sum_{k=1}^K w_{k,t})\right]^{+},
\end{align}
where $t_1$ is the number of iterations for the subgradient method, $\epsilon^{t_1}$ is the step size set in iteration $t_1$ and $[x]^+ = \max\{x,0\}$.

\subsection{Finding Optimal $(\boldsymbol{\alpha}^*, \boldsymbol{\beta}^*, \boldsymbol{\gamma}^*)$ for Given $(\boldsymbol{p}^*, \boldsymbol{w}^*)$}
After obtaining the optimal selection probability and bandwidth allocation ratio $(\boldsymbol{p}^*, \boldsymbol{w}^*)$, we now propose an algorithm to update $(\boldsymbol{\alpha}, \boldsymbol{\beta}, \boldsymbol{\gamma})$. To begin with, we define the following functions.
\begin{align}
    &\psi_{k,t} =  \alpha_{k,t}R_{k,t}^*-1, \\
    &\kappa_{k,t} = \beta_{k,t}R_{k,t}^*-p_{k,t}^*P_k S(1-\rho),  \\
    &\chi_{k,t} = \gamma_k - \frac{\rho T^2}{K(\sum_{t=0}^{T-1} p_{k,t}^*)^2}.
\end{align}
 According to \cite{jong2012efficient}, $(\boldsymbol{p}^*, \boldsymbol{w}^*)$ is also the optimal solution to problem (P2) if $\psi_{k,t}=\kappa_{k,t}=\chi_{k,t}=0$ for all $k\in \mathcal{K}$ and $t \in \mathcal{T}$. Otherwise, we update $(\boldsymbol{\alpha}, \boldsymbol{\beta}, \boldsymbol{\gamma})$ using the modified Newton method \cite{jong2012efficient} as follows.
\begin{align}
&\alpha_{k,t}^{t_2+1} = (1-\zeta^l)\alpha_{k,t}^{t_2} + \zeta^l \frac{1}{R_{k,t}^*}, \label{eqn:outer_1} \\
&\beta_{k,t}^{t_2+1} = (1-\zeta^l)\beta_{k,t}^{t_2} + \zeta^l \frac{p_{k,t}^*P_k S(1-\rho)}{R_{k,t}^*}, \label{eqn:outer_2}\\
&\gamma_{k,t}^{t_2+1} = (1-\zeta^l)\gamma_{k,t}^{t_2} + \zeta^l\frac{\rho T^2}{K(\sum_{t=0}^{T-1} p_{k,t}^*)^2},\label{eqn:outer_3}
\end{align}
where $t_2$ is the iteration index and $\zeta^l$ is the step size set at iteration $t_2$ via the following method. Let $l$ be the smallest integer among $l\in\{1,2,3,\cdots\}$ satisfying
\begin{align}\label{eqn:smallest}
\sum_{t=0}^{T-1}\sum_{k=1}^K&\Bigg\{\left|\psi_{k,t}\left((1-\zeta^{l})\alpha_{k,t}^{t_2} + \zeta^{l} \frac{1}{R_{k,t}^*}\right)\right|^2 \nonumber\\
&+\left| \kappa_{k,t} \left(( 1-\zeta^{l})\beta_{k,t}^{t_2} + \zeta^{l} \frac{p_{k,t}^*P_kS(1-\rho)}{R_{k,t}^*} \right)\right|^2 
\nonumber\\
&+\left|\chi_{k,t}\left((1-\zeta^{l})\gamma_{k,t}^{t_2} + \zeta^{l}\frac{\rho T^2}{K(\sum_{t=0}^{T-1} p_{k,t}^*)^2}\right) \right|^2 \Bigg\} \nonumber\\
&\leq (1-\varepsilon\zeta^{l})\sum_{t=0}^{T-1}\sum_{k=1}^K\big(|\psi_{k,t}(\alpha_{k,t}^{t_2})|^2\nonumber\\
&~~+|\kappa_{k,t}(\beta_{k,t}^{t_2})|^2+|\chi_{k,t}(\gamma_{k,t}^{t_2})|^2\big),
\end{align}
where $\varepsilon\in (0,1)$ and $\zeta\in(0,1)$.
\par
To summarize, we optimally solve the inner-layer convex problem (P3) and (P4) given $(\boldsymbol{\alpha}, \boldsymbol{\beta}, \boldsymbol{\gamma})$, and then solve the outer-layer problem by the modified Newton method, where the two-layer problems are solved alternately in a loop to finally obtain the globally optimal solution to problem (P2), which is presented in Algorithm \ref{alg:1}.

\setlength{\textfloatsep}{20pt}%
\begin{algorithm}[htb]
\caption{Optimal Algorithm for Problem (P2)}
\begin{algorithmic}[1]\label{alg:1}
\STATE \textbf{Initialize} $(\boldsymbol{\alpha}, \boldsymbol{\beta}, \boldsymbol{\gamma})$.
\REPEAT
\STATE \textbf{Initialize} $(\boldsymbol{p},\boldsymbol{w})$.
\REPEAT
\STATE Compute $p_{k,t}$ for given $(\boldsymbol{\alpha}, \boldsymbol{\gamma})$ with fixed $p_{k,j}$ $(\forall j\neq t)$ according to \eqref{eqn:bcd}.
\UNTIL $\boldsymbol{p}$ converges.
\REPEAT 
\STATE Compute $\boldsymbol{w}$ for given $(\boldsymbol{\alpha}, \boldsymbol{\beta}, \boldsymbol{v})$ according to \eqref{eqn:optimal_w}.
\STATE Update $\boldsymbol{v}$ based on the subgradient method in \eqref{eqn:update_v}.
\UNTIL $\boldsymbol{v}$ converges.
\STATE Update $\boldsymbol{\alpha}$, $\boldsymbol{\beta}$ and $\boldsymbol{\gamma}$ by \eqref{eqn:outer_1}, \eqref{eqn:outer_2} and \eqref{eqn:outer_3}, respectively.
\UNTIL the conditions in \eqref{eqn:cont} are achieved.
\ENSURE $(\boldsymbol{p}^*, \boldsymbol{w}^*)$.
\end{algorithmic}
\end{algorithm}

\subsection{Online Implementation}
So far the proposed Algorithm 1 has found the optimal selection probability and bandwidth allocation ratio. However, since the calculation of the selection probability $p_{k,t}$ in \eqref{eqn:bcd} needs the probabilities of other rounds, Algorithm 1 can be operated by an offline fashion. To address this issue, we extend the algorithm to an online one. Specifically, we assume that the selection probability of each client $p_{k,t}$ is identical in all rounds, i.e., $p_{k,t}=p_k$ for all round $t$. Note that this assumption is widely used in the literature \cite{shi2020joint, wang2021quantized}. Then, problem (P1) can be rewritten as
\begin{align}
    (\text{P1}^{\prime}):\quad \min_{\boldsymbol{p}^{\prime}, \boldsymbol{w}^{\prime}} \quad & \frac{\rho}{K}\sum_{k=1}^K\left(\frac{1}{p_{k}}\right)^2 \nonumber\\
    &~~+ (1-\rho)T\sum_{k=1}^K \frac{p_{k} P_k S}{w_{k} W\log\left(1+\frac{P h_{k}}{w_{k} W N_0}\right)} \\
    {\rm{s.t.}} \quad&\sum_{k=1}^{K} w_{k} \leq 1 ,  \label{eqn:st1}\\
    &0 \leq w_{k} \leq 1,  \quad\quad~\forall k\in\mathcal{K}, \label{eqn:st2}\\
    &\lambda \leq p_{k} \leq 1,  \quad\quad~\forall k\in\mathcal{K}, \label{eqn:st3}
\end{align}
where $\boldsymbol{p}^{\prime} = \{p_{1}, p_{2},\cdots,p_{K}\}$ and $\boldsymbol{w}^{\prime} = \{w_{1}, w_{2},\cdots,w_{K}\}$. Problem (P1$^\prime$) has the same structure as problem (P1), thus we can extend the proposed Algorithm 1 to solve problem (P1$^\prime$). Specifically, problem (P1$^\prime$) can be transformed into a parameterized subtractive-form problem with given auxiliary parameters $(\boldsymbol{\alpha}, \boldsymbol{\beta}, \boldsymbol{\gamma})$ as
\begin{align}
    &\quad \min_{(\boldsymbol{p}^{\prime}, \boldsymbol{w}^{\prime})\in \mathcal{F}} \quad  \frac{\rho}{K}\sum_{k=1}^K\left(\frac{1}{ p_{k}^2}-\gamma_k\right) + \nonumber\\
    &\sum_{k=1}^K\alpha_{k}\left(p_{k}P_kST(1-\rho)-\beta_{k}w_{k} W\log\left(1+\frac{P_k h_{k}}{w_{k} W N_0}\right)\right),
\end{align}
where $\mathcal{F}$ is the set of $(\boldsymbol{p}^{\prime}, \boldsymbol{w}^{\prime})$ satisfying constraints \eqref{eqn:st1}-\eqref{eqn:st3}. Then the optimal $w_{k}^*$ can be given by \eqref{eqn:optimal_w}, where the subscript $t$ in \eqref{eqn:optimal_w} is omitted. The optimal $p_k$ can be given by 
\begin{align}
    p_{k}^* = \left[\left(\frac{2\rho}{K\alpha_{k}P_k ST(1-\rho)}\right)^{\frac{1}{3}}\right]_{\lambda}^1.
\end{align}
Here the optimal $p_k^*$ is calculated only using the information in the current round, so the algorithm can be implemented online.

\section{Experimental Results}
\label{sec:Experimental_Results}
In this section, we conduct experiments to evaluate the performance of our proposed scheme.
\subsection{Experiment Settings}
\label{Experiment_Settings}
\emph{Wireless Network Settings:} We consider a cell network with a radius of $1000$ m. The server is located at the center of the network and the clients are randomly and uniformly distributed in the cell, where the path loss between each client and the server is $128.1 + 37.6 \log_{10}(r)$ (in dB) according to \cite{abetaevolved}, where $r$ is the distance from the client to the server in kilometers. The client transmit power is set uniformly to $0.2$ W. We consider orthogonal uplink channel access with total bandwidth $W=5$ MHz and the power spectrum density of the additive Gaussian noise $N_0 = -174$ dBm/Hz. Unless mentioned otherwise, the main parameters of wireless networks used in the experiments are summarized in Table \ref{parameter_setting_table}.
\begin{table}[t]
\caption{Parameters of the Wireless Network}
\label{parameter_setting_table}
\vspace{-0.1in}
\centering
\begin{tabular}{|c|c|}
\hline
Parameter                             & Value                      \\ \hline
Number of clients, $K$                & $10$                         \\ \hline
Radius of the cell network, $R$       & $1000$ m                      \\ \hline
System bandwidth, $W$                 & $5$ MHz                      \\ \hline
Path loss from client to server       & $128.1 + 37.6 \log_{10}r_{[km]}$ dB \\ \hline
Client transmit power, $P_k$          & $0.2$ W                       \\ \hline
Channel noise power density, $N_0$    & $-174$ dBm/Hz             \\ \hline
\end{tabular}
\end{table}
\par
\emph{FL Settings:} We consider the learning task of training a classifier to evaluate the proposed scheme on the MNIST dataset and the CIFAR-10 dataset, respectively. MNIST is a database of handwritten digits containing $60,000$ training examples and $10,000$ test examples, the examples contain $10$ categories from the numbers ``0" to ``9". CIFAR10 is a commonly image classification database containing $50,000$ training examples and $10,000$ test examples, the examples contain $10$ different objects. With the same setup as in \cite{zhu2019broadband}, we assume that each client has an equal number of training examples and the local training datasets do not overlap with each other. Also, we consider the non-IID data distribution to better model the data heterogeneity in asynchronous FL. Specifically, we first divide the dataset into $10$ data blocks according to the label, and then further divide each data block into $\frac{d\cdot K}{10}$ shards, and finally each client is assigned with $d$ shards with different labels. And the non-IID level of local datasets can be controlled by $d$, and the smaller the value of $d$, the more heterogeneous the data distribution is. 
We train a simple multilayer perceptron (MLP) model consisting of one hidden layer with $200$ nodes for the MNIST dataset and this neural network’s model size is $S = 6.37 \times 10^6$ bits. We set the batch size to $10$, the learning rate to $0.01$ and the clients to perform $5$ local iteration per communication round.
We use AlexNet \cite{krizhevsky2012imagenet} to train the CIFAR-10 dataset and this model size is $4.57 \times 10^8$ bits. We set the batchsize to $128$, the learning rate to $0.01$ and the clients to perform $1$ local iteration per round.
\par
\emph{Benchmarks:} We introduce three benchmark schemes to demonstrate the effectiveness of the proposed scheme: 1) \emph{Random scheme:} All clients decide whether to communicate with the server with the same probability $p$. 2) \emph{Greedy scheme \cite{9170917, 9685090}:} In each communication round, it selects the top $k$ clients with largest channel gain to participate in the model training. 3) \emph{Age-based scheme \cite{yang2020age}:} In each communication round, it selects $k$ clients in a sequential and alternating manner. Note that the value of $k$ is set to the same as the number of the selected clients in the proposed scheme for fair comparison.

\subsection{Evaluation of the Proposed Scheme}
\begin{figure}[t]
\centering
 \includegraphics[width=9.0 cm]{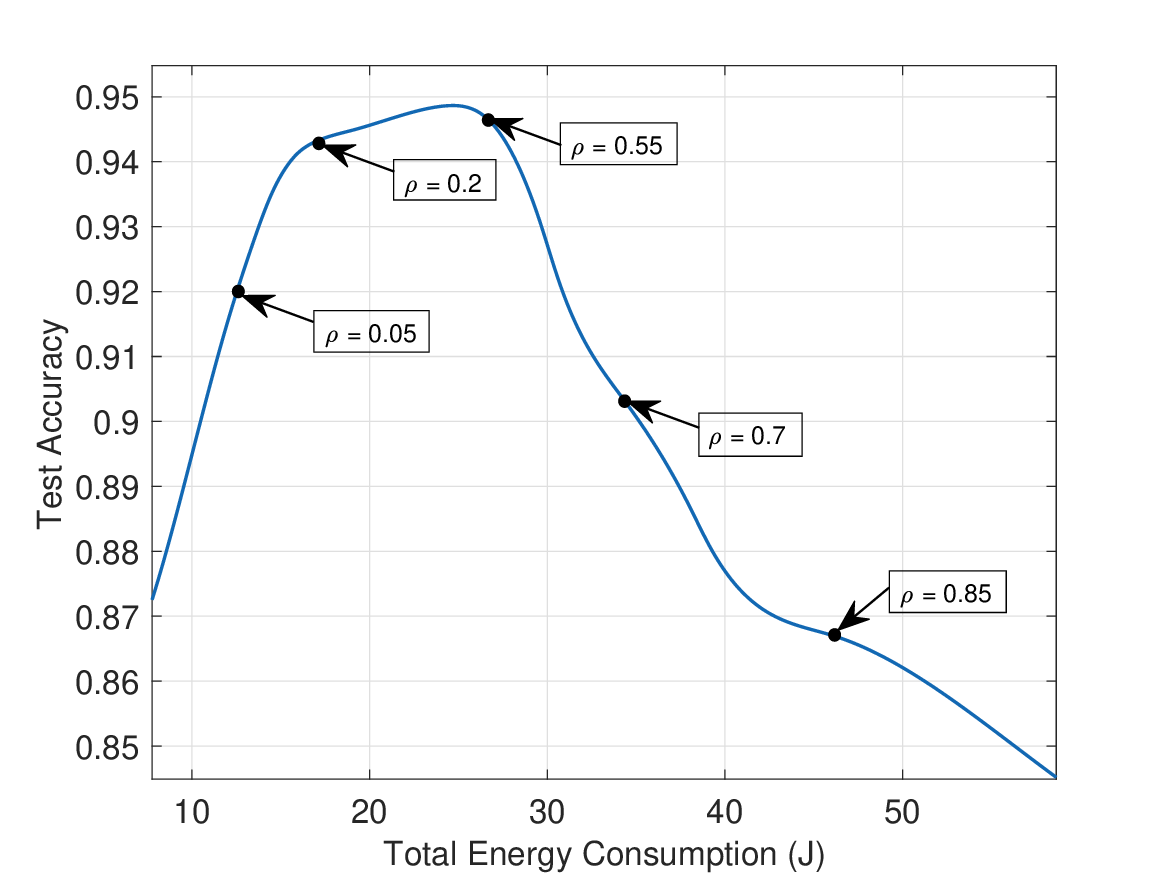}
\vspace{-0.1in}
\caption{Test accuracy v.s. total energy consumption where $\rho$ varies from $0.01$ to $0.9$.}
\label{fig:1}
\end{figure}

\begin{figure}[t]
\centering
 \includegraphics[width=8.8cm]{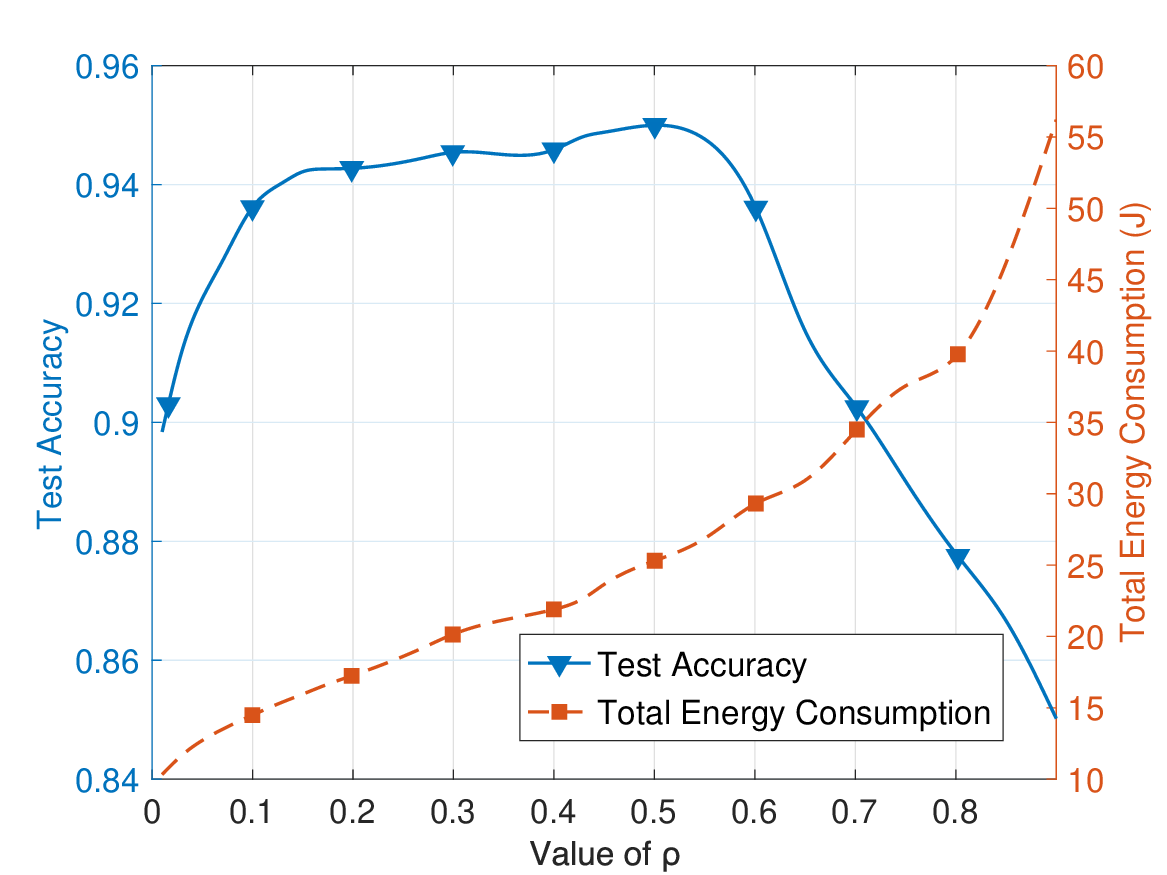}
\vspace{-0.1in}
\caption{Test accuracy and total energy consumption v.s. value of $\rho$.}
\label{fig:2}
\end{figure}
As mentioned in Section \ref{sec:system_model_and_problem_formulation}, we use the coefficient $\rho$ to weigh the convergence performance and energy consumption in our optimization objective, with a larger $\rho$ focusing more on convergence and a smaller $\rho$ on energy, thus we first discuss the impact of $\rho$ on the model performance. We train a model on the MNIST dataset at a non-IID level of $d=5$ (each local dataset has $5$ different categories). Fig. \ref{fig:1} shows the relationship between the test accuracy and the total energy consumption after $50$ communication rounds by varying $\rho$ from $0.01$ to $0.9$. We observe that the test accuracy rapidly increase with respect to total energy consumption when $\rho$ is small, i.e., when the $\rho$ is between $0.01$ and $0.1$. It indicates that the focus of the optimization changes from minimizing the energy consumption to increasing the convergence performance as $\rho$ increases. Fig. \ref{fig:2} shows the test accuracy and total energy consumption after $50$ communication rounds with respect to the value of $\rho$. From Fig. \ref{fig:2}, as the value of $\rho$ increases from $0.01$ to $0.1$, more clients participate in the model training, which although increases the total energy consumption, it enables the global model acquiring more information from clients and thus converge faster and obtain a higher test accuracy. However, when $\rho$ is larger, i.e., when $\rho$ increases from $0.1$ to $0.9$, the test accuracy first remains stable and then gradually decreases. This is due to the data heterogeneity of the clients and leads to large differences between local updates of the clients. Thus the global model obtained after aggregation deviates from the optimal direction of global loss decrease,  and thus results in a decrease in test accuracy. Therefore, in practice, we can select $\rho$ in a suitable range, such as among $0.01$ to $0.1$, where we can select a larger $\rho$ for the focus of model convergence, and at the same time, we can select a smaller $\rho$ to meet a lower energy budget.
\subsection{Performance Comparison}
\begin{figure}[t]
\centering
 \includegraphics[clip, viewport= 30 0 420 320, width=8.6cm]{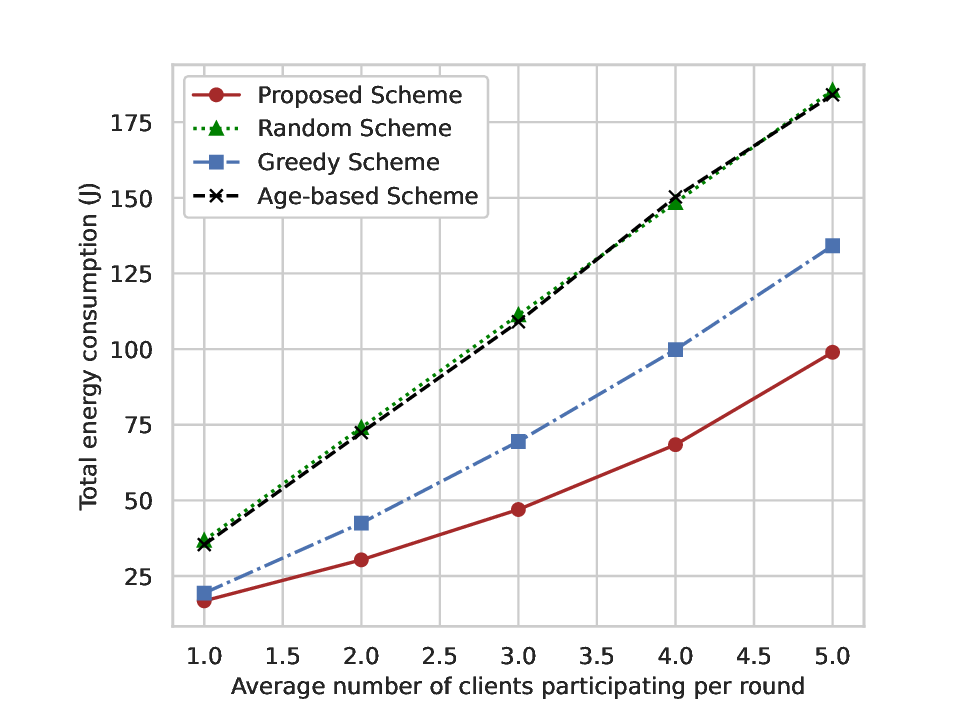}
\caption{Total energy consumption v.s. the average number of clients involved per round.}
\label{fig:3_0}
\end{figure}

\begin{figure}[t]
\centering
 \includegraphics[clip, viewport= 30 0 420 320, width=8.6cm]{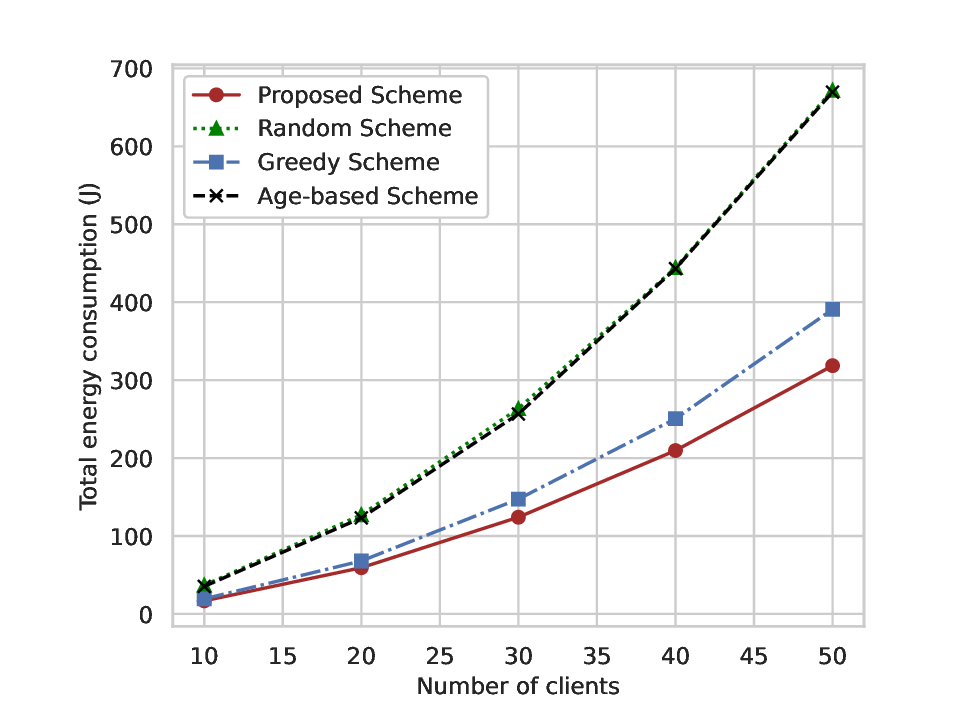}
\caption{Total energy consumption v.s. the number of clients.}
\label{fig:3_5}
\end{figure}

\begin{figure}[t]
  \centering 

  \subfigure[Test accuracy v.s. energy consumption when the average number of clients participating per round is 1 with MNIST dataset.]{ 
    \includegraphics[clip, viewport= 30 0 420 360,width=4.1cm]{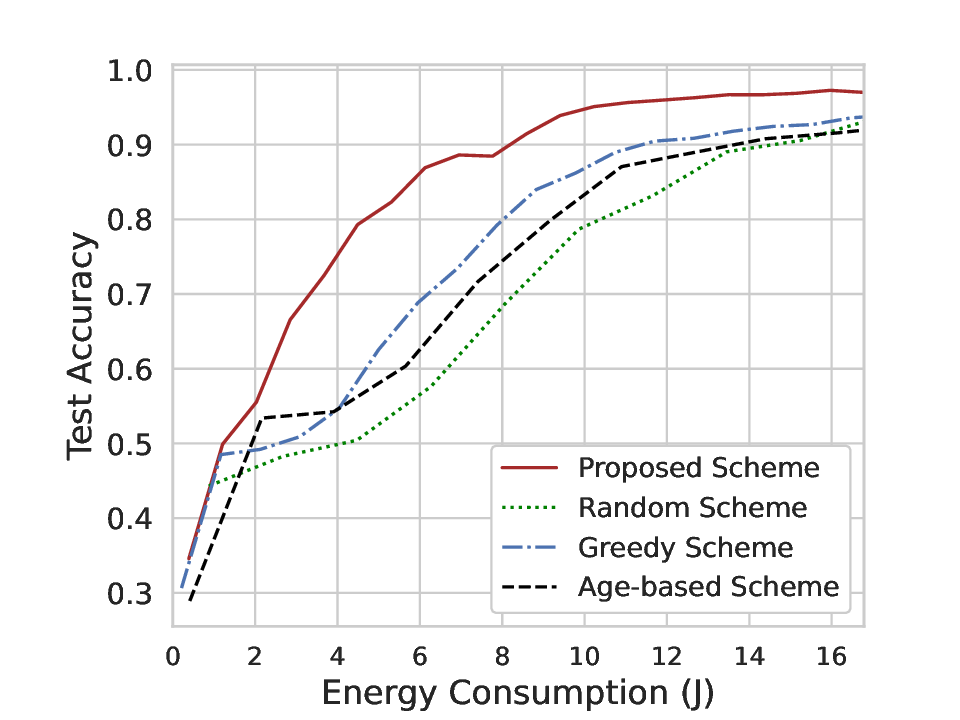} 
  } 
  \centering
  \subfigure[Test accuracy v.s. energy consumption when the average number of clients participating per round is 2 with MNIST dataset.]{ 
    \includegraphics[clip, viewport= 30 0 420 360, width=4.1cm]{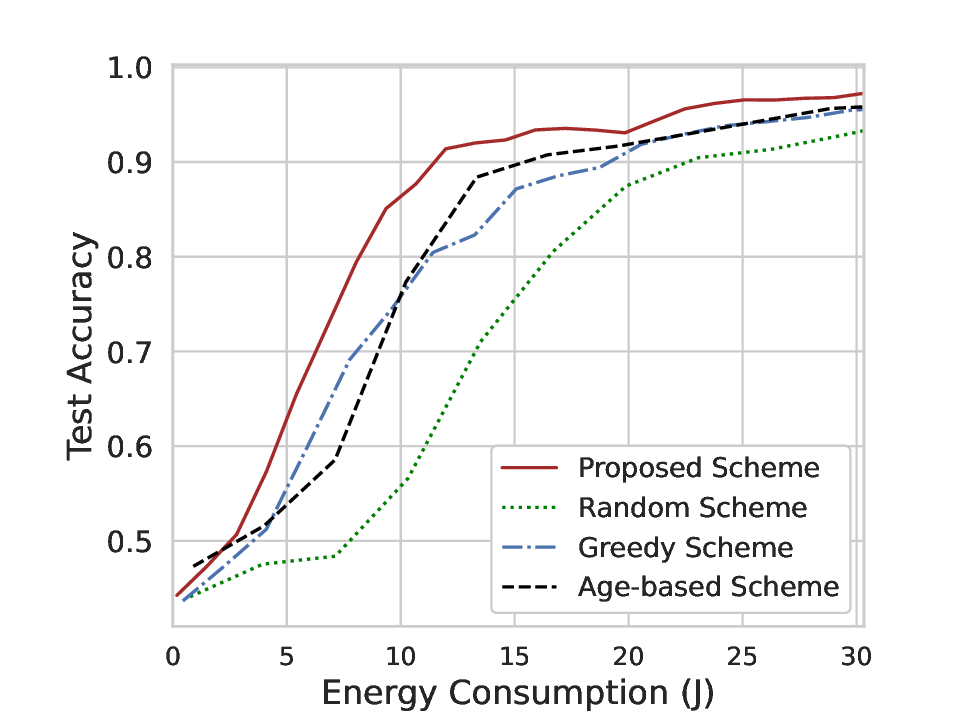} 
  } 
  \\
  \centering 
  \subfigure[Test accuracy v.s. energy consumption when the average number of clients participating per round is 1 with CIFAR-10 dataset.]{ 
    \includegraphics[clip, viewport= 30 0 420 360,width=4.1cm]{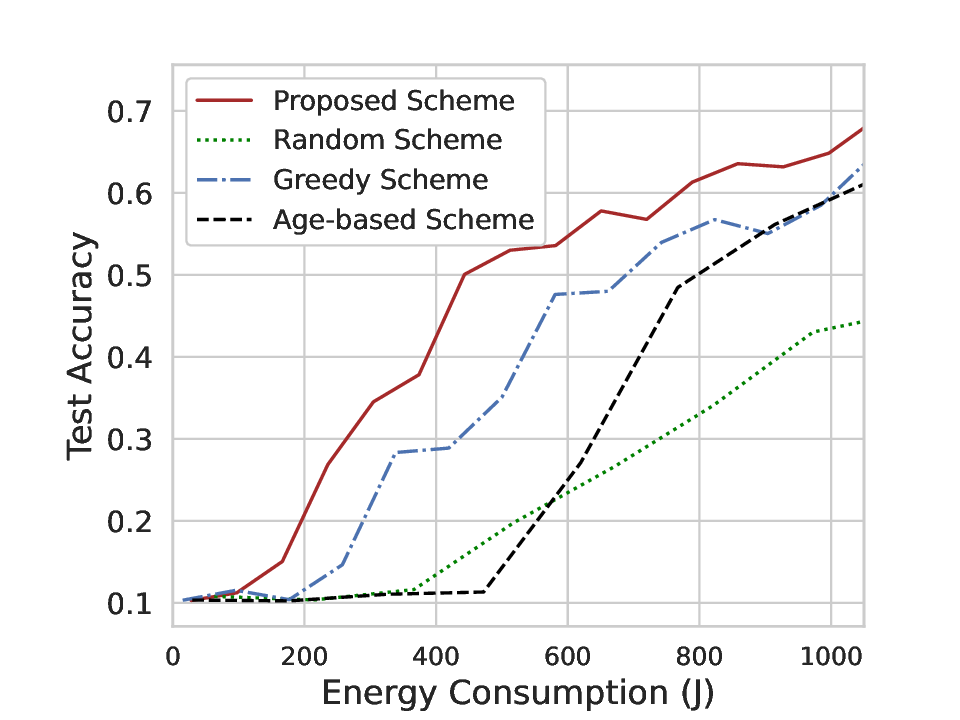} 
  } 
  \centering
  \subfigure[Test accuracy v.s. energy consumption when the average number of clients participating per round is 2 with CIFAR-10 dataset.]{ 
    \includegraphics[clip, viewport= 30 0 420 360,width=4.1cm]{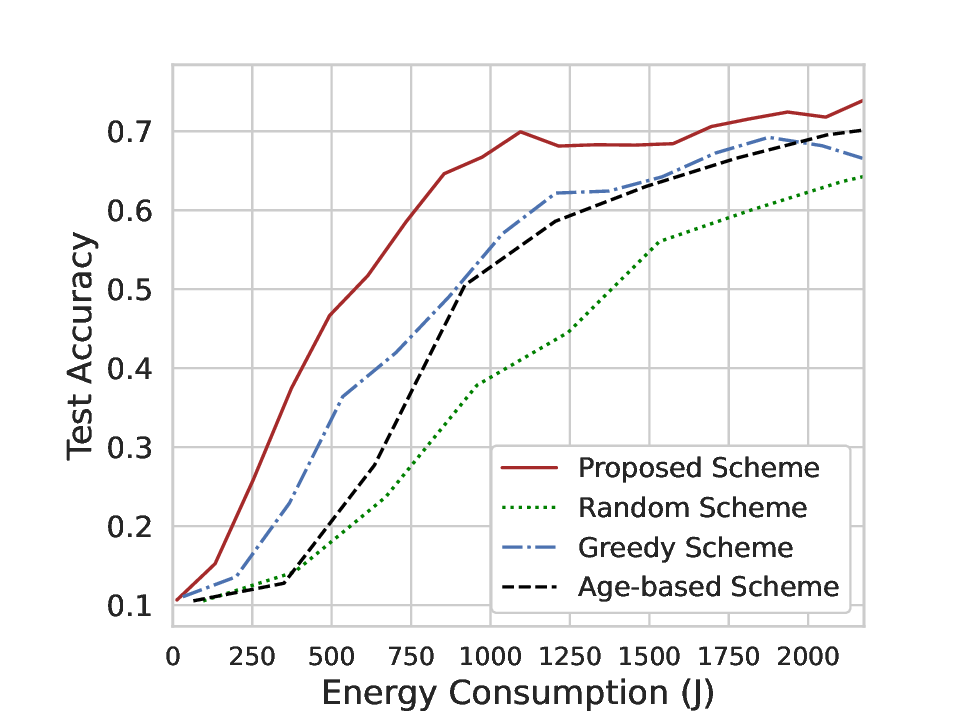} 
  } 
  \caption{The asynchronous FL convergence performance under different schemes.} 
\label{fig:3_1}
\end{figure}

\begin{figure}[t]
  \centering 

  \subfigure[Test accuracy v.s. energy consumption when the number of clients is 20 with MNIST dataset.]{ 
    \includegraphics[clip, viewport= 30 0 420 360,width=4.1cm]{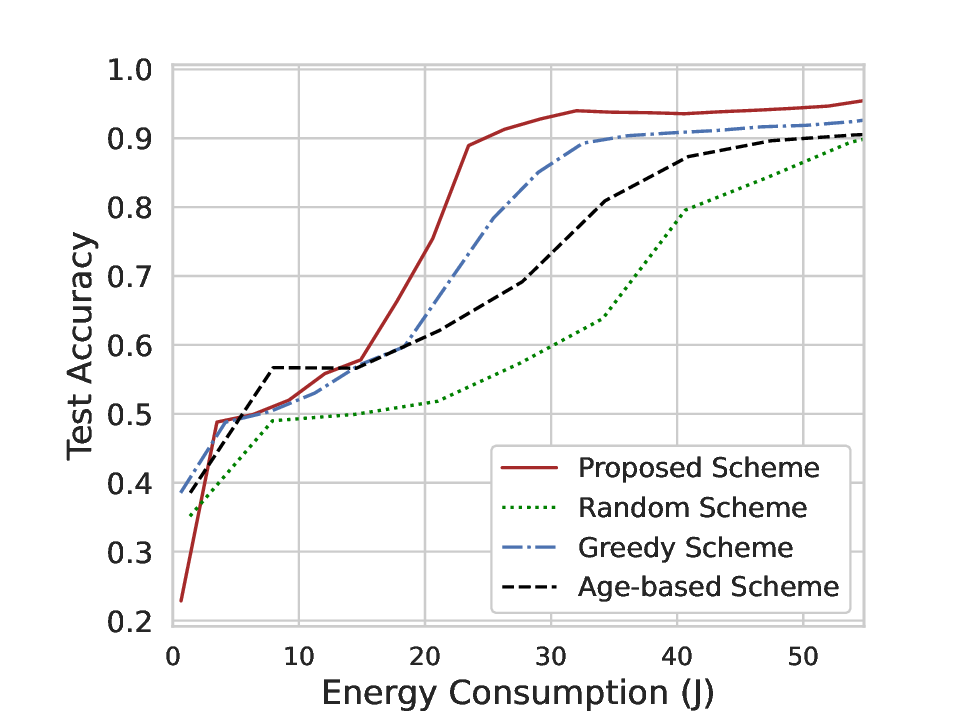} 
  } 
  \centering
  \subfigure[Test accuracy v.s. energy consumption when the number of clients is 30 with MNIST dataset.]{ 
    \includegraphics[clip, viewport= 30 0 420 360,width=4.1cm]{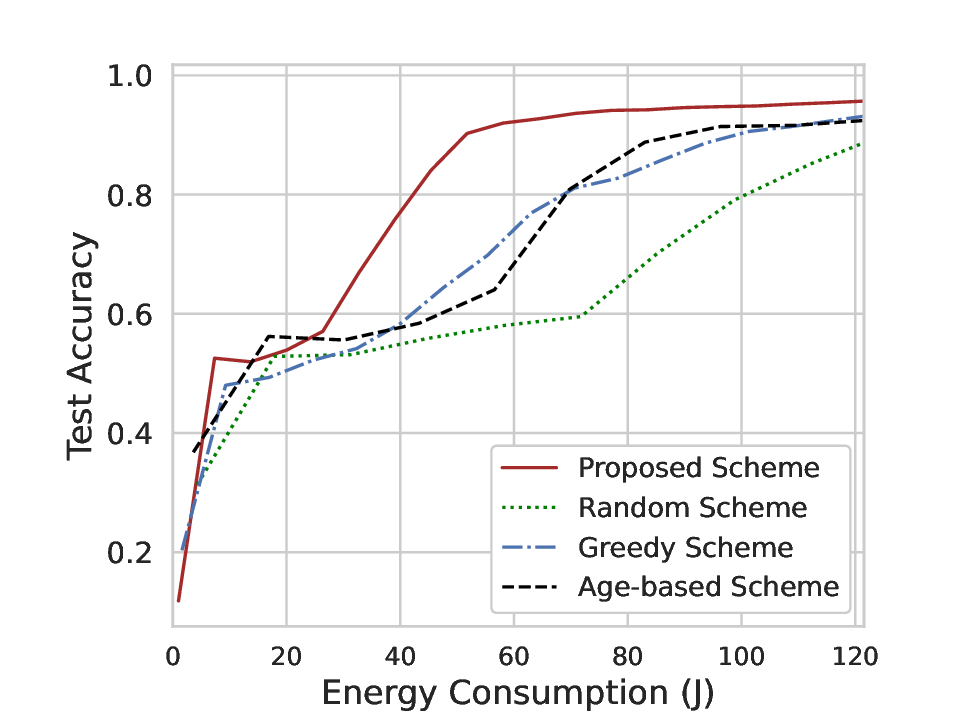} 
  } 
  \\
  \centering 
  \subfigure[Test accuracy v.s. energy consumption when the number of clients is 20 with CIFAR-10 dataset.]{ 
    \includegraphics[clip, viewport= 30 0 420 360,width=4.1cm]{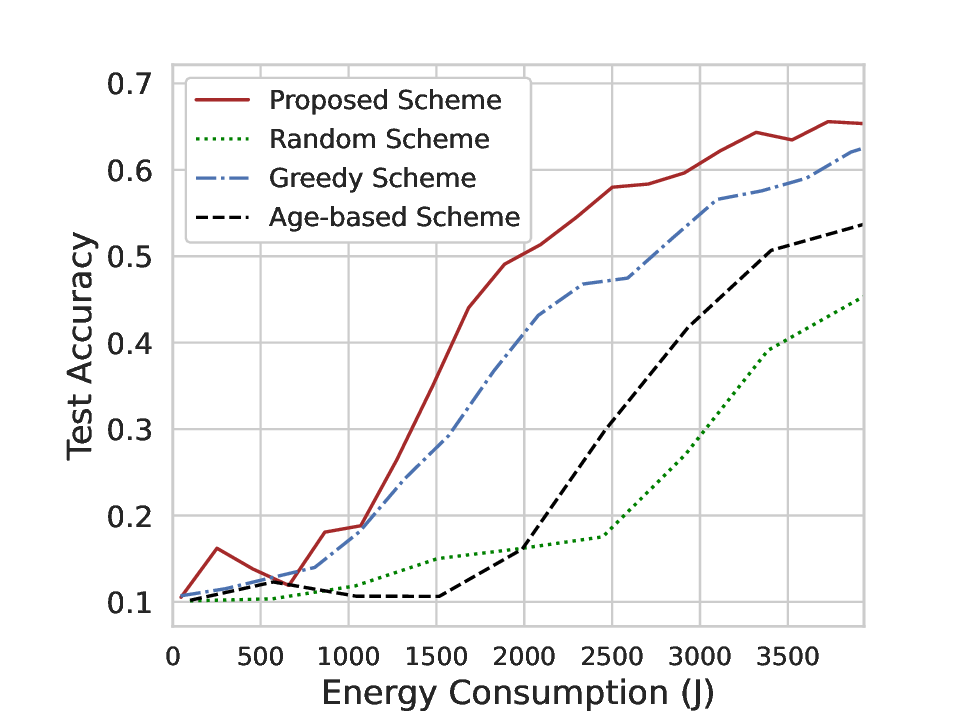} 
  } 
  \centering
  \subfigure[Test accuracy v.s. energy consumption when the number of clients is 30 with CIFAR-10 dataset.]{ 
    \includegraphics[clip, viewport= 30 0 420 360,width=4.1cm]{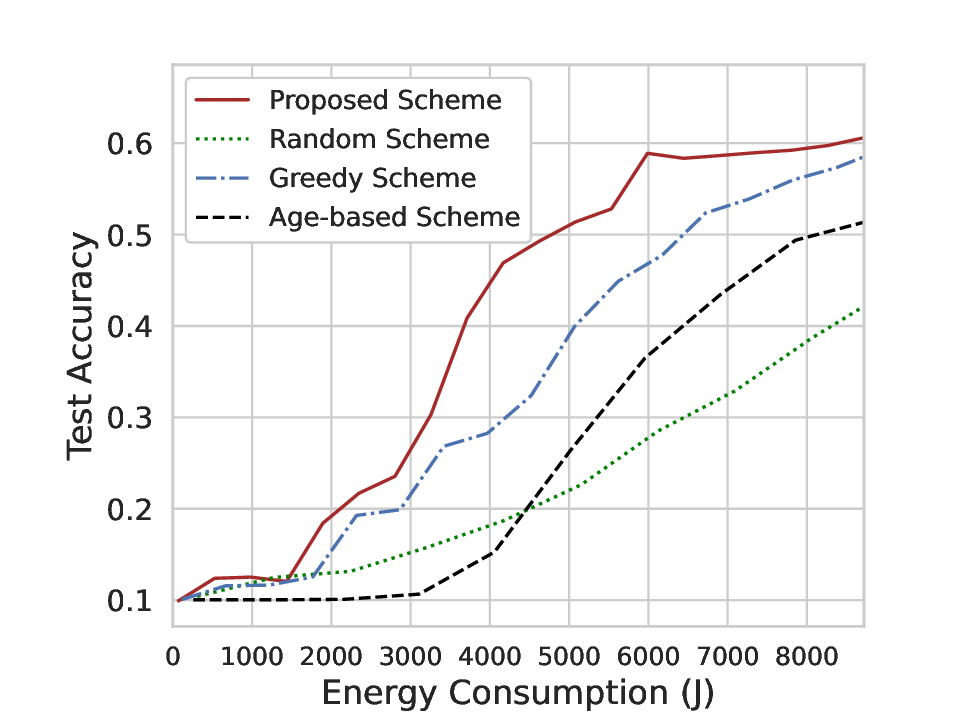} 
  } 
  \caption{The asynchronous FL convergence performance under different number of clients.} 
\label{fig:3_2}
\end{figure}
To show the effectiveness of the scheme, we compare the proposed scheme with several benchmark schemes introduced in Section \ref{Experiment_Settings}. 
Fig. \ref{fig:3_0} shows the variation of the total energy consumption over $100$ communication rounds with the average number of clients involved per round, where the number of clients is set to $10$. Fig. \ref{fig:3_5} shows the changes in total energy consumption over $100$ communication rounds as the number of clients varies, where the client participation rate is uniformly set to $0.1$. We can observe that the proposed scheme significantly reduces the energy consumption. Fig. \ref{fig:3_1} and Fig. \ref{fig:3_2} show the learning performance of four schemes under MNIST and CIFAR-10, where the non-IID level $d$ is set to $5$ and the total number of clients participating in the model training is set uniformly for fair comparison. Specifically, the average number of clients participating per round is set to $1$ or $2$ in Fig. \ref{fig:3_1}, and the client participation rate is uniformly set to $0.1$ in Fig. \ref{fig:3_2}. From Fig. \ref{fig:3_1} and Fig. \ref{fig:3_2}, we can observe that the random scheme has the worst performance and the proposed scheme has the highest accuracy for the same energy consumption. This indicates that the proposed scheme effectively reduces the energy consumption by joint bandwidth allocation and probabilistic client selection.

\subsection{Adaptability to Varying Network Conditions}\label{section_D}
\begin{figure}[t]
  \centering 
  \subfigure[Test accuracy v.s. energy consumption in Scenario 1 with MNIST dataset.]{ 
    \includegraphics[clip, viewport= 30 0 420 360,width=4.1cm]{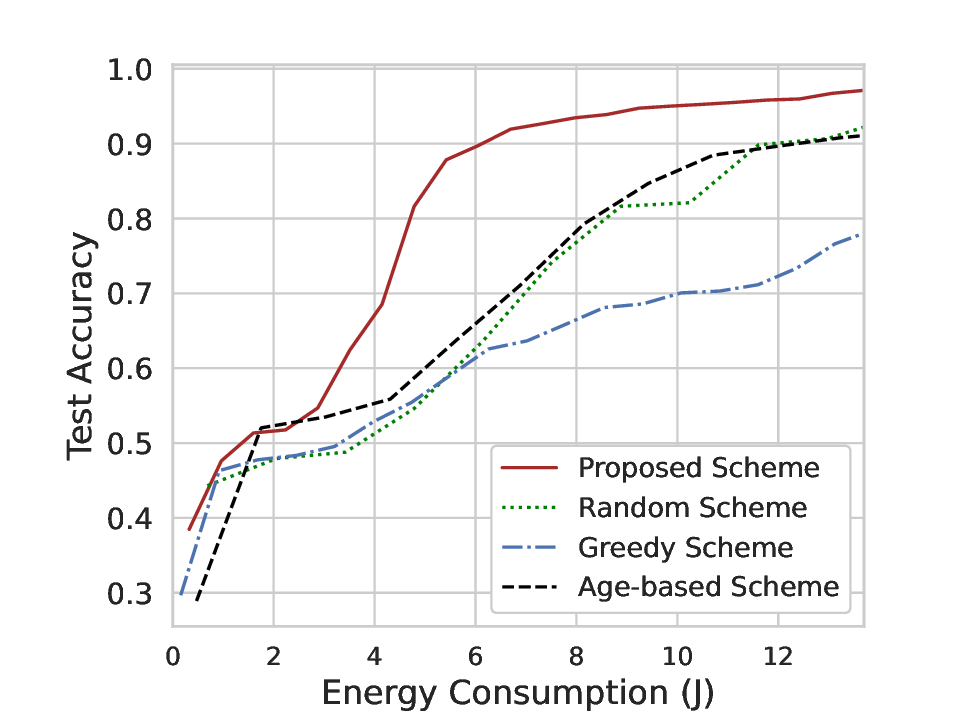} 
  } 
 \centering
  \subfigure[Test accuracy v.s. energy consumption in Scenario 2 with MNIST dataset.]{ 
    \includegraphics[clip, viewport= 30 0 420 360,width=4.1cm]{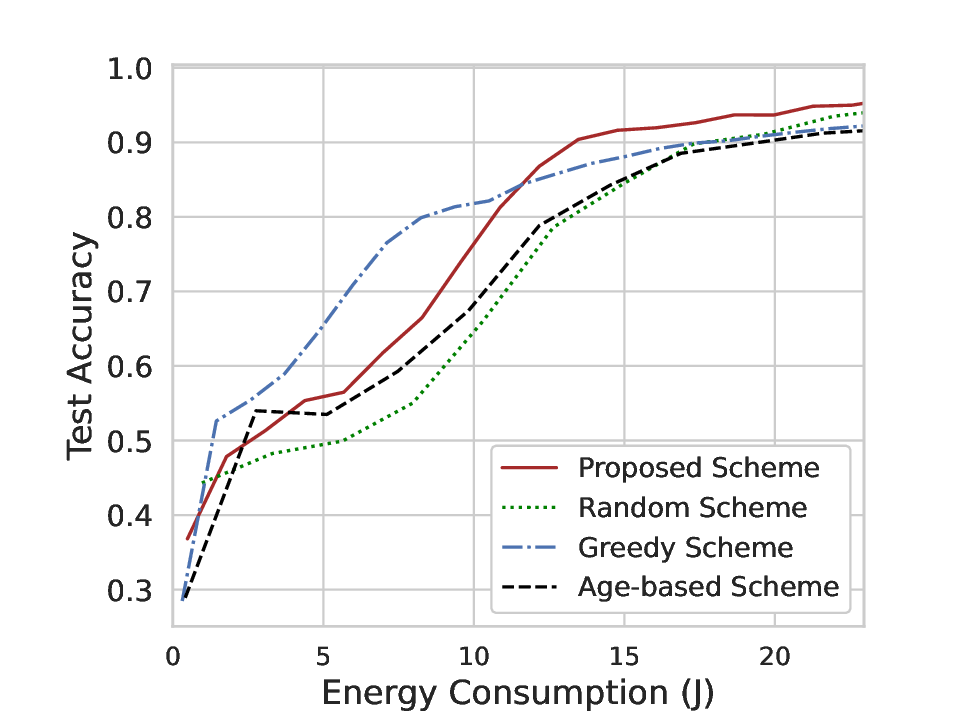} 
  } 
  \\
  \centering 
  \subfigure[Test accuracy v.s. energy consumption in Scenario 1 with CIFAR-10 dataset.]{ 
    \includegraphics[clip, viewport= 30 0 420 360,width=4.1cm]{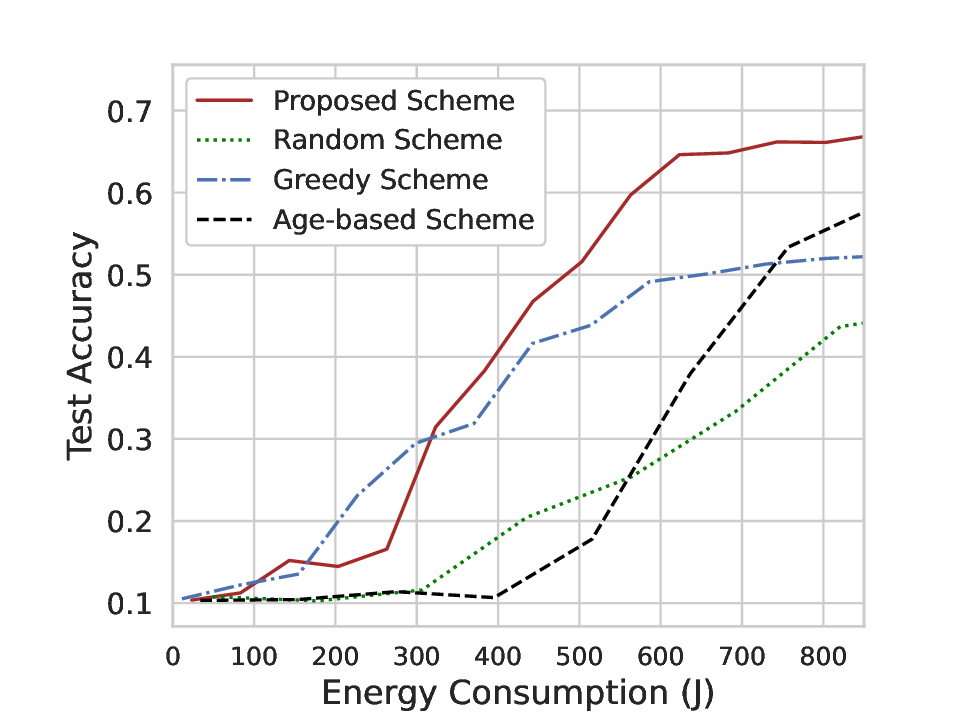} 
  } 
\centering
  \subfigure[Test accuracy v.s. energy consumption in Scenario 2 with CIFAR-10 dataset.]{ 
    \includegraphics[clip, viewport= 30 0 420 360,width=4.1cm]{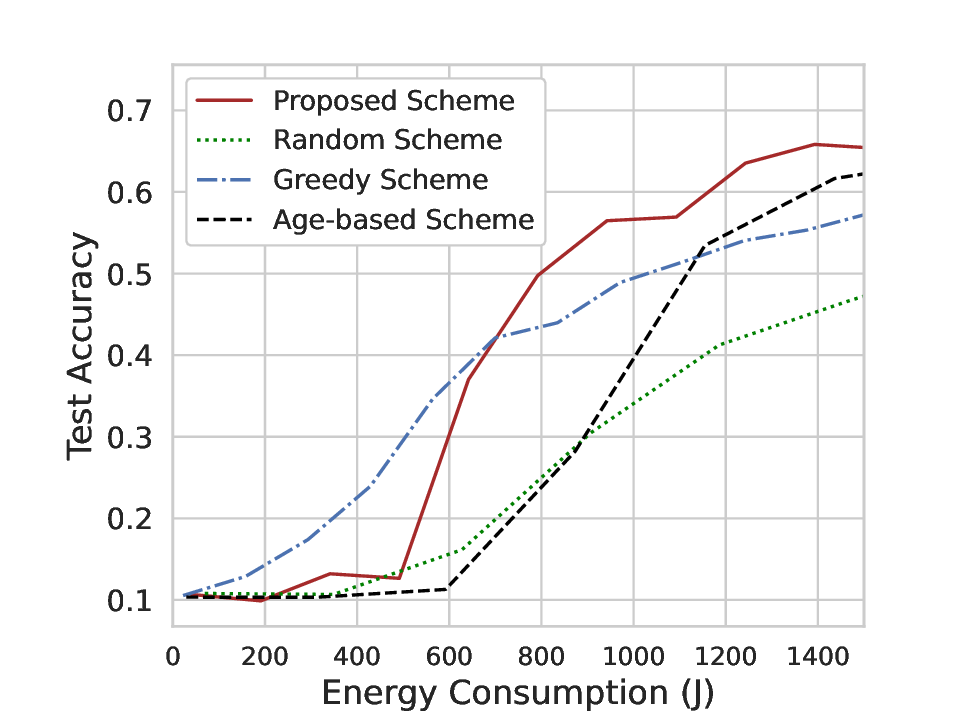} 
  } 
  \caption{The asynchronous FL convergence performance under different schemes for the two scenarios.} 
\label{fig:4}
\end{figure}

\begin{figure}[t]
  \centering 
  \subfigure[The per-client total energy consumption in scenario 1.]{ 
    \hspace{-1cm}
    \includegraphics[width=10.6cm]{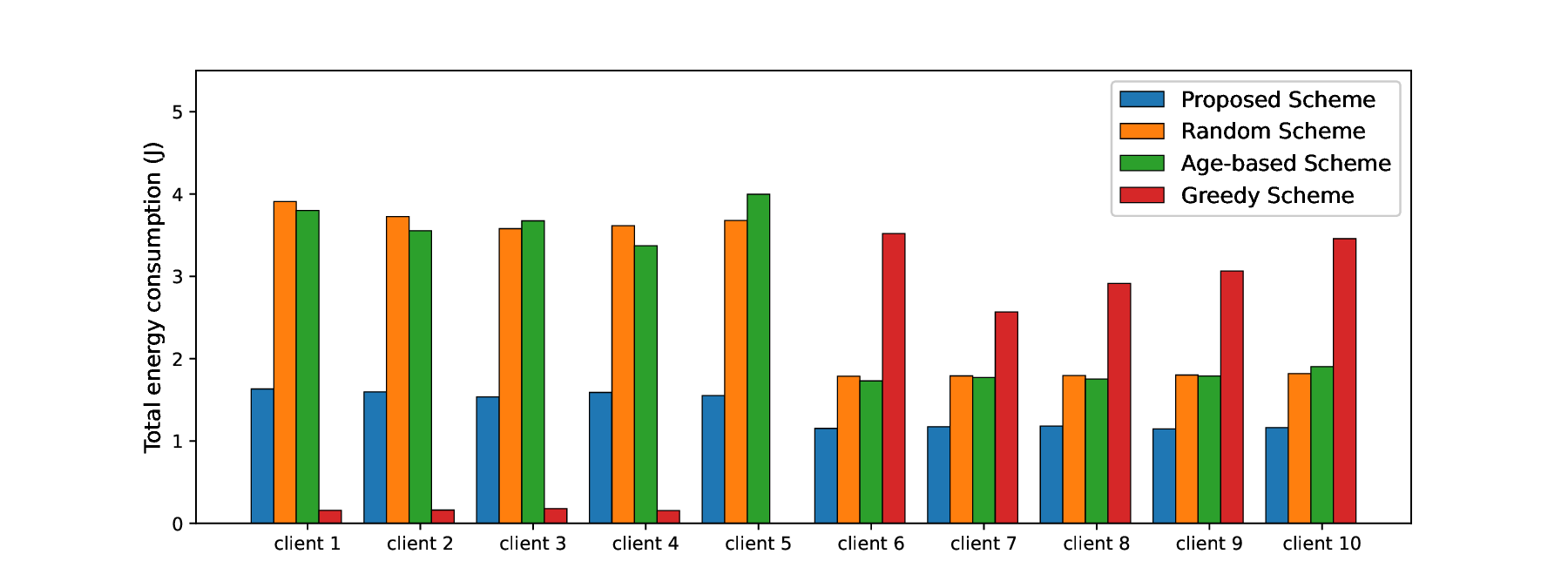} 
  }
   \centering
  \subfigure[The per-client total energy consumption in scenario 2.]{ 
    \hspace{-1cm}
    \includegraphics[width=10.6cm]{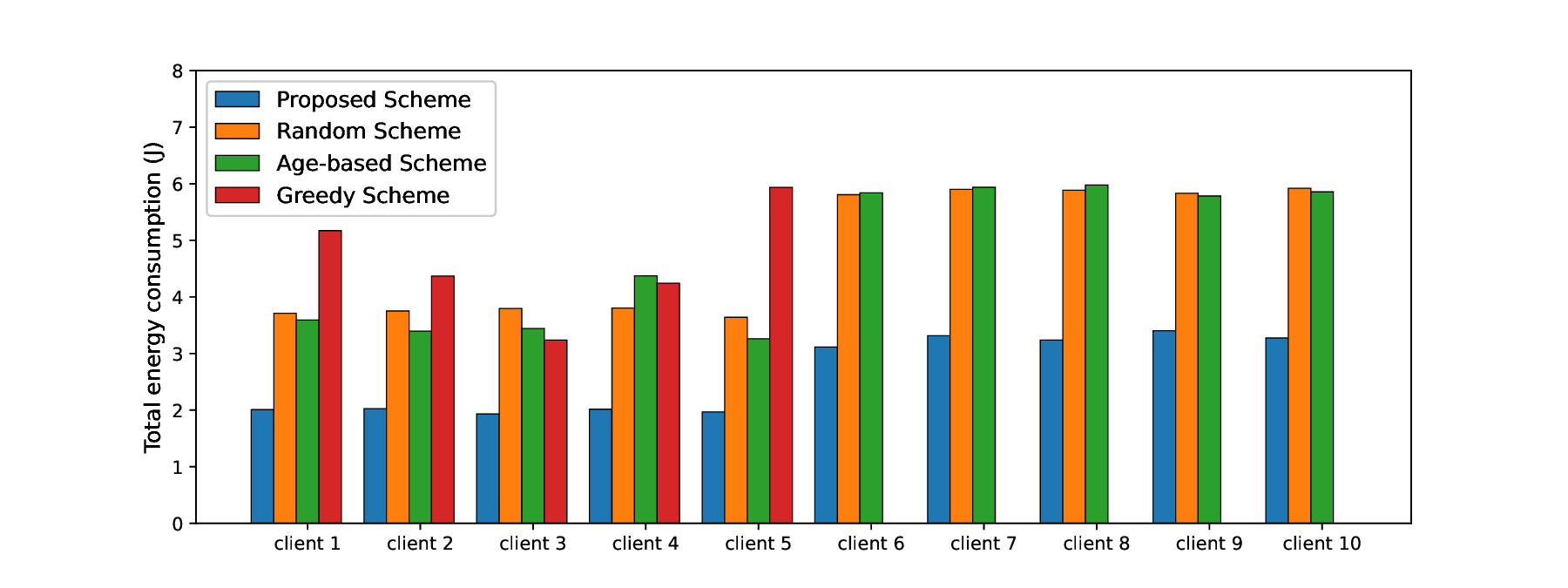} 
  } 
  \caption{The per-client total energy consumption under different schemes for the two scenarios with MNIST dataset.} 
\label{fig:5}
\end{figure}

In the above experiments, we consider that the locations of clients are distributed randomly and thus clients can be fairly selected with best channel conditions in the greedy scheme. However, in practice, clients may not be highly mobile, so here we consider two scenarios with extremely distributed clients. In Scenario 1, client 1 to client 5 are distributed in an area with a radius of $100$ m to $200$ m from the server and the rest of the clients are randomly distributed in the cell, mimicking a scenario where some clients are always near around the server. In Scenario 2, client 1 to client 5 are distributed in an area with a radius of $900$ m to $1000$ m from the server and the rest of the clients are randomly distributed in the cell, mimicking a scenario where some clients always far from the server.
We train the model in two scenarios using the MNIST dataset and CIFAR-10 dataset and set the average number of clients participating per round is 1, where the non-IID level $d$ is set to 5. Fig. \ref{fig:4} shows the model accuracy for both scenarios. We can observe that the greedy scheme exhibits poor model performance and even lower accuracy than the random scheme in MNIST dataset for the same energy consumption in the end of the training process, while our proposed scheme still maintains the optimal accuracy. Fig. \ref{fig:5} shows the total energy consumption per client under the proposed scheme and two benchmark schemes for the two scenarios with MNIST dataset. From Fig. \ref{fig:5}, we can observe that the greedy scheme selects the extremely distributed clients multiple times in Scenario 1, while in Scenario 2, extremely distributed clients are completely ignored, indicating that the greedy scheme will only naively select the clients with good channel conditions, thus resulting in unfair client participation and making the global model drift towards local optimum. 
In contrast, the total energy consumption of each client in the proposed scheme is more balanced and similar to the random scheme and the age-based scheme in fairness. Meanwhile, since the proposed scheme involves channel-aware client selection, the per-client total energy consumption of the proposed scheme is much lower than that of the random scheme.

\section{Conclusion}
\label{sec:Conclusion}
In this paper, we studied asynchronous FL over wireless networks by joint probabilistic client selection and bandwidth allocation. A bound of convergence rate for asynchronous FL with probabilistic client selection was derived. Then, a joint optimization problem of probabilistic client selection and bandwidth allocation was formulated to trade off the convergence performance of the asynchronous FL and energy consumption of the clients. We proposed an iterative algorithm to find the globally optimal solution of the non-convex problem. Comprehensive experiments showed that the proposed scheme can improve both model convergence and energy efficiency.

\section*{Appendix}
\appendices
\subsection{Proof of Lemma \ref{lemma1}} 
We start with two definitions according to \cite{avdiukhin2021federated}. 
\begin{definition}
\textbf{Global model parameters sequence}
\begin{align}
    \boldsymbol{x}_{t} = \boldsymbol{x}_0 - \frac{\eta}{K}\sum_{k=1}^K\sum_{t^{\prime}=0}^{\pi_{k,t}} \nabla F_{k,t^{\prime}},
\end{align}
where $K$ is the number of clients and $\pi_{k,t}$ is the last round where the client $k$ communicates with the server before round $t$.
\end{definition}
\begin{definition}
\textbf{Virtual sequence}
\begin{align}
    \boldsymbol{z}_{t}=\boldsymbol{x}_0 - \sum_{t^{\prime}=0}^{t-1}\frac{\eta}{K}\sum_{k=1}^K \nabla F_{k,t^{\prime}}.
\end{align}
\end{definition}
\par
Then we have the following lemma.
\begin{lemma}\label{lemma2}
The distance of the global model parameters and virtual sequence can be bounded as 
\begin{align}\label{eqn:lemma2_1}
\mathbb{E}[\|\boldsymbol{z}_T-\boldsymbol{x}_T\|^2] &\leq \frac{\eta^2}{K}\mathbf{G}_{max}^2\sum_{k=1}^K \Delta_k^2,
\end{align}
\begin{align}\label{eqn:lemma2_2}
\mathbb{E}[\|\boldsymbol{x}_{k,\pi_{k,T}}-\boldsymbol{x}_{k,T}\|^2] &\leq \eta^2\mathbf{G}_{max}^2\Delta_k^2,
\end{align}
\begin{align}
\mathbb{E}[\|\boldsymbol{x}_{k,t}-\boldsymbol{z}_t\|^2] &\leq \left[\frac{6\eta^2}{K}\mathbf{G}_{max}^2\sum_{k=1}^K \Delta_k^2 + 3\eta^2\mathbf{G}_{max}^2\Delta_k^2\right],
\end{align}
where $\mathbb{E}[\cdot]$ is the expectation of the randomness results from stochastic gradients.
\end{lemma}
We first prove Lemma \ref{lemma2} using Jensen's inequality.
Due to the convexity of L2 norm $\|\nabla F_{k,t}\|^2$, we bound $\mathbb{E}[\|\boldsymbol{z}_T-\boldsymbol{x}_T\|^2]$ using Jensen’s inequality as
\begin{align}
    \mathbb{E}[\|\boldsymbol{z}_T-\boldsymbol{x}_T\|^2] &= \mathbb{E}\left[\left\|\frac{1}{K}\sum_{k=1}^K\left(\sum_{t=\pi_{k,T}}^{T-1}\eta \nabla F_{k,t}\right)\right\|^2\right] \\
     &\leq \frac{\eta^2}{K}\sum_{k=1}^K (T-\pi_{k,t})\sum_{t=\pi_{k,T}}^{T-1}\mathbb{E}[\|\nabla F_{k,t}\|^2].
\end{align}
Since $\pi_{k,t}$ is the last round where the client $k$ communicates with the server before round $t$ and $\Delta_k$ is the maximum communication interval between each client $k$ and the server, we have $T-\Delta_k \leq \pi_{k,t}$ and the term $\mathbb{E}[\|\boldsymbol{z}_T-\boldsymbol{x}_T\|^2]$ can be further bounded as
\begin{align}
    \mathbb{E}[\|\boldsymbol{z}_T-\boldsymbol{x}_T\|^2] 
    &\leq \frac{\eta^2}{K}\sum_{k=1}^K\Delta_k\sum_{t=T-\Delta_k}^{T-1}\mathbb{E}[\|\nabla F_{k,t}\|^2] \\
    & \leq  \frac{\eta^2}{K}\mathbf{G}_{max}^2 \sum_{k=1}^K\Delta_k^2.\label{eqn:lemma2_44}
\end{align}
\par
Then we rewrite the term $\mathbb{E}[\|\boldsymbol{x}_{k,\pi_{k,T}}-\boldsymbol{x}_{k,T}\|^2]$ as
\begin{align}
    \mathbb{E}&[\|\boldsymbol{x}_{k,\pi_{k,T}}-\boldsymbol{x}_{k,T}\|^2] = \mathbb{E}\left[\left\|\sum_{t=\pi_{k,T}}^{T-1}\eta\nabla F_{k,t}\right\|^2\right]
    \nonumber\\
    &\leq \eta^2\Delta_k\sum_{t=T-\Delta_k}^{T-1}\mathbb{E}[\|\nabla F_{k,t}\|^2] 
    \leq \eta^2 \mathbf{G}_{max}^2\Delta_k^2.
\end{align}
Define a local model parameter $\boldsymbol{x}_{k,t^\prime}$, where the client $k$ will communicate with the server in round $t^\prime$, thus we have $\boldsymbol{x}_{k,t^\prime}=\boldsymbol{x}_{t^\prime}$. Then we further rewrite the term $\mathbb{E}[\|\boldsymbol{x}_{k,t}-\boldsymbol{z}_t\|^2]$ as 
\begin{align}
    \mathbb{E}[\|\boldsymbol{x}_{k,t}-\boldsymbol{z}_t\|^2] &= \mathbb{E}[\|\boldsymbol{z}_t-\boldsymbol{x}_t-\boldsymbol{x}_{k,t}+\boldsymbol{x}_{k,t^\prime}-\boldsymbol{x}_{t^\prime}+\boldsymbol{x}_t\|^2]\\
    &\leq 3\mathbb{E}[\|\boldsymbol{z}_t-\boldsymbol{x}_t\|^2+\|\boldsymbol{x}_{k,t^\prime}-\boldsymbol{x}_{k,t}\|^2 \nonumber\\
    &~~+ \|\boldsymbol{x}_t-\boldsymbol{x}_{t^\prime}\|^2], \label{eqn:lemma2_3}
\end{align}
Using \eqref{eqn:lemma2_1} and \eqref{eqn:lemma2_2}, we can further bound \eqref{eqn:lemma2_3} as
\begin{align}
    &\leq 3\left[\frac{\eta^2}{K}\mathbf{G}_{max}^2 \sum_{k=1}^K\Delta_k^2 + \eta^2 \mathbf{G}_{max}^2\Delta_k^2 + \frac{\eta^2}{K}\mathbf{G}_{max}^2 \sum_{k=1}^K\Delta_k^2\right]\\
    &\leq \frac{6\eta^2}{K}\mathbf{G}_{max}^2 \sum_{k=1}^K\Delta_k^2 +3\eta^2 \mathbf{G}_{max}^2\Delta_k^2,
\end{align}
which completes the proof of Lemma \ref{lemma2}.
\par
Then we start to prove Lemma \ref{lemma1}. When assumption \ref{assumption1} holds, the decrease of the loss function for the virtual sequence $\{\boldsymbol{z}_t\}_{t=0}^{T-1}$ can be bounded as
\begin{align}\label{eqn:decrease_z}
    \mathbb{E}[f(\boldsymbol{z}_{t+1})]\leq& \mathbb{E}[f(\boldsymbol{z}_t)] - \mathbb{E}\left[\left<\nabla f(\boldsymbol{z}_t), \frac{\eta}{K}\sum_{k=1}^K\nabla F_{k,t}\right>\right] \nonumber\\
    &+ \frac{L}{2} \mathbb{E}\left[\left\|\frac{\eta}{K}\sum_{k=1}^K\nabla F_{k,t} \right \|^2\right].
\end{align}
The last term can be further bounded as 
\begin{align}
\frac{L}{2} \mathbb{E}&\left[\left\|\frac{\eta}{K}\sum_{k=1}^K\nabla F_{k,t} \right \|^2\right] \nonumber\\
&= \frac{\eta^2L}{2}\mathbb{E}\left[\left\|\frac{1}{K}\sum_{k=1}^K\left(\nabla f_k(\boldsymbol{x}_{k,t})+(\nabla F_{k,t} - \nabla f_k(\boldsymbol{x}_{k,t}))\right)\right\|^2\right] \\
&\leq \underbrace{\eta^2L\mathbb{E}\left[\left\|\frac{1}{K}\sum_{k=1}^K\nabla f_k(\boldsymbol{x}_{k,t})\right\|^2\right]}_{(a)} \nonumber\\
&~~+ \underbrace{\eta^2L\mathbb{E}\left[\left\|\frac{1}{K}\sum_{k=1}^K(\nabla F_{k,t}-\nabla f_k(\boldsymbol{x}_{k,t}))\right\|^2\right]}_{(b)}.
\end{align}
We further bound the term $(a)$ as
\begin{align}
(a) &\leq \eta^2L\mathbb{E}\left[\left\|\frac{1}{K}\sum_{k=1}^K(\nabla f_k(\boldsymbol{x}_{k,t})-\nabla f_k(\boldsymbol{z}_t)+\nabla f_k(\boldsymbol{z}_t))\right\|^2\right] \\
&\leq \frac{2\eta^2L}{K}\sum_{k=1}^K\mathbb{E}[\|\nabla f_k(\boldsymbol{x}_{k,t})-\nabla f_k(\boldsymbol{z}_t)\|^2]\nonumber\\
&~~+2\eta^2L\mathbb{E}\left[\left\|\frac{1}{K}\sum_{k=1}^K\nabla f_k(\boldsymbol{z}_t)\right\|^2\right] \label{eqn:a_1}\\
&\leq \frac{2\eta^2L}{K}\sum_{k=1}^K L^2 \mathbb{E} [\|\boldsymbol{x}_{k,t}-\boldsymbol{z}_t\|^2] + 2\eta^2L\mathbb{E}[\|\nabla f(\boldsymbol{z}_t)\|^2] \label{eqn:a_2}\\
&\leq \frac{2\eta^2L^3}{K}\sum_{k=1}^K \left[\frac{6\eta^2}{K}\mathbf{G}_{max}^2 \sum_{k=1}^K\Delta_k^2 +3\eta^2 \mathbf{G}_{max}^2\Delta_k^2\right] \nonumber\\
&~~+ 2\eta^2L\mathbb{E}[\|\nabla f(\boldsymbol{z}_t)\|^2] \label{eqn:a_3}\\
&\leq 2\eta^2L^3\left[\frac{6\eta^2}{K}\mathbf{G}_{max}^2 \sum_{k=1}^K\Delta_k^2 + \frac{3\eta^2}{K}\mathbf{G}_{max}^2\sum_{k=1}^K\Delta_k^2\right]\nonumber\\
&~~+2\eta^2L\mathbb{E}[\|\nabla f(\boldsymbol{z}_t)\|^2]\\
&\leq \frac{18\eta^4L^3\mathbf{G}_{max}^2}{K}\sum_{k=1}^K\Delta_k^2 + 2\eta^2L\mathbb{E}[\|\nabla f(\boldsymbol{z}_t)\|^2],
\end{align}
where \eqref{eqn:a_1} is obtained by Cauchy-Schwarz inequality and Jensen's inequality, \eqref{eqn:a_2} is obtained using the property of L-Lipschitz continuous and \eqref{eqn:a_3} is bounded according to Lemma \ref{lemma2}.
\par
Then we bound the term $(b)$ as 
\begin{align}
(b) &\leq \frac{\eta^2L}{K}\sum_{k=1}^K\mathbb{E}[\|\nabla F_{k,t} - \nabla  f_k(\boldsymbol{x}_{k,t})\|^2] \label{eqn:b_1} \\ 
&\leq \eta^2L\sigma^2, \label{eqn:b_2}
\end{align}
where \eqref{eqn:b_1} is obtained by Jensen's inequality and \eqref{eqn:b_2} is obtained using Assumption \ref{assumption1}.
\par
Then we further bound the term $-\mathbb{E}\left[\left<\nabla f(\boldsymbol{z}_t), \frac{\eta}{K}\sum_{k=1}^K\nabla F_{k,t}\right>\right]$ in \eqref{eqn:decrease_z} as
\begin{align}
&-\mathbb{E}\left[\left<\nabla f(\boldsymbol{z}_t), \frac{\eta}{K}\sum_{k=1}^K\nabla F_{k,t}\right>\right] \\
=&  -\eta\mathbb{E}\left[\left<\nabla f(\boldsymbol{z}_t), \frac{1}{K}\sum_{k=1}^K\nabla f_k(\boldsymbol{z}_t)\right>\right] \nonumber\\
&~~- \eta\mathbb{E}\left[\left<\nabla f(\boldsymbol{z}_t), \frac{1}{K}\sum_{k=1}^K(\nabla F_{k,t}-\nabla f_k(\boldsymbol{z}_t))\right>\right]\\
\leq& -\eta\mathbb{E}[\|\nabla f(\boldsymbol{z}_t)\|^2] +\frac{\eta}{2}\mathbb{E}[\|\nabla f(\boldsymbol{z}_t)\|^2] \nonumber\\
&~~+ \frac{\eta}{2}\mathbb{E}\left[\left\|\frac{1}{K}\sum_{k=1}^K(\nabla F_{k,t}-\nabla f_k(\boldsymbol{z}_t))\right\|^2\right] \label{eqn:hip_1}\\
\leq& -\frac{\eta}{2}\mathbb{E}[\|\nabla f(\boldsymbol{z}_t)\|^2] + \frac{\eta}{2K}\sum_{k=1}^K\mathbb{E}[\|\nabla F_{k,t}-\nabla f_k(\boldsymbol{z}_t)\|^2], \label{eqn:hip_2}
\end{align}
where \eqref{eqn:hip_1} is obtained by Jensen's inequality. Then the last term of \eqref{eqn:hip_2} can be further bounded as
\begin{align}
&\frac{\eta}{2K}\sum_{k=1}^K\mathbb{E}[\|\nabla F_{k,t}-\nabla f_k(\boldsymbol{z}_t)\|^2]\\
\leq& \frac{\eta}{2K}\sum_{k=1}^K\mathbb{E}[\|\nabla F_{k,t}-\nabla f_k(\boldsymbol{x}_{k,t})+\nabla f_k(\boldsymbol{x}_{k,t})-\nabla f_k(\boldsymbol{z}_t)\|^2] \\
\leq& \frac{\eta}{K}\sum_{k=1}^K\mathbb{E}[\|\nabla F_{k,t}-\nabla f_k(\boldsymbol{x}_{k,t})\|^2] \nonumber\\
&+ \frac{\eta}{K}\sum_{k=1}^K\mathbb{E}[\|\nabla f_k(\boldsymbol{x}_{k,t})-\nabla f_k(\boldsymbol{z}_t)\|^2] \label{eqn:hip_3}\\ 
\leq& \eta\sigma^2 + \frac{\eta L^2}{K}\sum_{k=1}^K\mathbb{E}[\|\boldsymbol{x}_{k,t}-\boldsymbol{z}_t\|^2] \label{eqn:hip_4}\\
\leq& \eta\sigma^2 + \frac{\eta L^2}{K}\sum_{k=1}^K\left[\frac{6\eta^2}{K}\mathbf{G}_{max}^2 \sum_{k=1}^K\Delta_k^2 + 3\eta^2\mathbf{G}_{max}^2\sum_{k=1}^K\Delta_k^2\right] \label{eqn:hip_5}\\
\leq& \eta\sigma^2 + \frac{9\eta^3L^2\mathbf{G}_{max}^2}{K}\sum_{k=1}^K\Delta_k^2,
\end{align}
where \eqref{eqn:hip_3} is obtained by Jensen's inequality, \eqref{eqn:hip_4} is obtained by Assumption \ref{assumption1} and \eqref{eqn:hip_5} is obtained by Lemma \ref{lemma2}. Combining the results above, the decrease of the loss function for the virtual sequence $\{\boldsymbol{z}_t\}_{t=0}^{T-1}$ can be bounded as
\begin{align}
 \mathbb{E}[f(\boldsymbol{z}_{t+1})]\leq& \mathbb{E}[f(\boldsymbol{z}_t)] - (\frac{1}{2}-2\eta L)\eta\mathbb{E}[\|\nabla f(\boldsymbol{z}_t)\|^2] \nonumber\\
&+ \eta \sigma^2(1+\eta L) + \frac{18\eta^4L^3\mathbf{G}_{max}^2}{K}\sum_{k=1}^K\Delta_k^2 \nonumber\\
&+ \frac{9\eta^3L^2\mathbf{G}_{max}^2}{K}\sum_{k=1}^K\Delta_k^2. \label{eqn:decrease_2}
\end{align}
Assume that the learning rate $\eta$ satisfies the constraint $\eta \leq \frac{1}{4L}$ such that \eqref{eqn:decrease_2} can be rewritten as
\begin{align}
(\frac{1}{2}-2\eta L)&\eta\mathbb{E}[\|\nabla f(\boldsymbol{z}_t)\|^2] \leq  \mathbb{E}[f(\boldsymbol{z}_t)]-\mathbb{E}[f(\boldsymbol{z}_{t+1})] \nonumber\\
&+ \frac{9(2\eta L+1)\eta^3L^2\mathbf{G}_{max}^2}{K}\sum_{k=1}^K\Delta_k^2 + \eta \sigma^2(1+\eta L).
\end{align}
Taking the sum over all rounds we have
\begin{align}
\frac{1}{T}\sum_{t=0}^{T-1} \mathbb{E}[\|\nabla f(\boldsymbol{z}_t)\|^2] \leq& \frac{f_{max}}{(\frac{1}{2}-2\eta L)T\eta} \nonumber\\
&+ \frac{9(2\eta L+1)\eta^3L^2\mathbf{G}_{max}^2\sum_{k=1}^K\Delta_k^2}{(\frac{1}{2}-2\eta L)K\eta} \nonumber\\
&+ \frac{\eta \sigma^2(1+\eta L)}{(\frac{1}{2}-2\eta L)\eta}, 
\end{align}
where $f_{max} = f(\boldsymbol{z}_0)-f(\boldsymbol{z}^*) = f(\boldsymbol{x}_0)-f(\boldsymbol{x}^*)$. Assume that $\eta \leq \frac{1}{8L}$, we have
\begin{align}
\frac{1}{T}\sum_{t=0}^{T-1} \mathbb{E}&[\|\nabla f(\boldsymbol{z}_t)\|^2] \nonumber\\
&\leq \frac{4f_{max}}{T\eta}+\frac{45}{K}\eta^2L^2\mathbf{G}_{max}^2\sum_{k=1}^K\Delta_k^2+\frac{9\sigma^2}{2}.\label{eqn:bound_1}
\end{align}
While the term $\mathbb{E}[\|\nabla f(\boldsymbol{z}_t)\|^2]$ can be bounded with $\mathbb{E}[\|\nabla f(\boldsymbol{x}_t)\|^2]$ as
\begin{align}
\mathbb{E}&[\|\nabla f(\boldsymbol{x}_t)\|^2] \nonumber\\
&\leq 2\mathbb{E}[\|\nabla f(\boldsymbol{z}_t)\|^2] +2\mathbb{E}[\|\nabla f(\boldsymbol{z}_t)-\nabla f(\boldsymbol{x}_t)\|^2] \\
&\leq 2\mathbb{E}[\|\nabla f(\boldsymbol{z}_t)\|^2] + 2L^2\mathbb{E}[\|\boldsymbol{z}_t-\boldsymbol{x}_t\|^2] \\
&\leq 2\mathbb{E}[\|\nabla f(\boldsymbol{z}_t) \|^2] + \frac{2}{K}\eta^2L^2\mathbf{G}_{max}^2\sum_{k=1}^K\Delta_k^2, \label{eqn:z_t_x}
\end{align}
Substituting \eqref{eqn:z_t_x} into \eqref{eqn:bound_1}, we can obtain the claimed bound,
\begin{align}
\frac{1}{T}\sum_{t=0}^{T-1} \mathbb{E}&[\|\nabla f(\boldsymbol{x}_t)\|^2] \nonumber\\
&\leq \frac{8f_{max}}{T\eta}+\frac{92}{K}\eta^2L^2\mathbf{G}_{max}^2\sum_{k=1}^K\Delta_k^2+9\sigma^2.
\end{align}

\subsection{Proof of Lemma \ref{lemma3}}
 According to $\Delta_k'$ defined in \eqref{eqn:delta_pik_app}, \eqref{eqn:conv_app} can be rewritten as $O\left(\frac{1}{K}\sum_{k=1}^K(\Delta_k')^2\right)$. We assume that the total number of communications between the server and clients is fixed, i.e., $\sum_{k=1}^K \frac{1}{\Delta_k'} = C$, where $C$ is a constant. According to Cauchy-Schwarz inequality, we can obtain the following inequality:
    \begin{align}
        \sum_{k=1}^K \left( \frac{1}{\sqrt{\Delta_k'}} \sqrt{\Delta_k'} \right)^2  &\leq  
        \sum_{k=1}^K\left( \frac{1}{\sqrt{\Delta_k'}}\right)^2  \sum_{k=1}^K \left( \sqrt{\Delta_k'}\right)^2, \\
        \frac{K}{C} &\leq \sum_{k=1}^K \Delta_k',\label{eqn:appendix_1}
    \end{align}
    where the equality is achieved when there exists a constant $\Delta$ such that $\Delta_k'=\Delta$ for all $k$. Then we can obtain the following inequality according to the arithmetic mean-quadratic mean inequality:
    \begin{align}
        \frac{\sum_{n=1}^K \Delta_k'}{K} \leq \sqrt{\frac{\sum_{n=1}^K (\Delta_k')^2}{K}},\label{eqn:appendix_2}
    \end{align}
     where the equality is achieved when there exists a constant $\Delta$ such that $\Delta_k'=\Delta$ for all $k$. Substituting \eqref{eqn:appendix_1} into \eqref{eqn:appendix_2}, we can obtain 
     \begin{align}
         \frac{1}{C} \leq \frac{\sum_{k=1}^K \Delta_k'}{K} \leq \sqrt{\frac{\sum_{n=1}^K (\Delta_k')^2}{K}},
     \end{align}
     which be simplified as
     \begin{align}
         \frac{K}{C^2} \leq \sum_{n=1}^K (\Delta_k')^2,
     \end{align}
     where the equality is achieved when there exists a constant $\Delta$ such that $\Delta_k'=\Delta$ for all $k$. Therefore, $\frac{1}{K}\sum_{k=1}^K(\Delta_k')^2$ achieves the lower bound when $\Delta_k'=\Delta$ for all $k$, which yields to the same transmission period and thus fair participation of clients. 

\subsection{Proof of Theorem \ref{theorem2}}
To prove this theorem, we first introduce the auxiliary variables $\boldsymbol{\beta} = \{\boldsymbol{\beta}_1,\boldsymbol{\beta}_2,\cdots,\boldsymbol{\beta}_K\}$ with $\boldsymbol{\beta}_k = \{\beta_{k,0}, \beta_{k,1},\cdots,\beta_{k,T-1}\}$ and $\boldsymbol{\gamma} = \{\gamma_1, \gamma_2,\cdots,\gamma_K\}$. Then problem (P1) can be transformed equivalently as
\begin{align}\label{eqn:t2proof}
 \min_{(\boldsymbol{p}, \boldsymbol{w}) \in \mathcal{F}} \quad & \sum_{k=1}^K \gamma_k + \sum_{t=0}^{T-1}\sum_{k=1}^K \beta_{k,t} \\
    {\rm{s.t.}} \quad& \frac{\rho T^2}{K(\sum_{t=0}^{T-1} p_{k,t})^2} \leq \gamma_k, 
     \quad\quad\quad~~\forall t\in\mathcal{T}  \label{eqn:t2proof_st1}\\
    &\frac{(1-\rho)p_{k,t} P_k S}{w_{k,t} W\log\left(1+\frac{P_k h_{k,t}}{w_{k,t} W N_0}\right)} \leq \beta_{k, t},\nonumber\\
&\quad\quad\quad\quad\quad\quad\quad\quad\quad\quad\forall t\in\mathcal{T},\forall k\in\mathcal{K}\label{eqn:t2proof_st2}
\end{align}
where $\mathcal{F}$ denotes the set of $(\boldsymbol{p}, \boldsymbol{w})$ satisfying the convex constraints \eqref{eqn:p1_2}-\eqref{eqn:p1_4}. Clearly, the optimal $(\boldsymbol{p}^*, \boldsymbol{w}^*, \boldsymbol{\beta}^*, \boldsymbol{\gamma}^*)$ to problem \eqref{eqn:t2proof} satisfies the following conditions:
\begin{align}\label{eqn:t2proof2}
\left\{
\begin{aligned}
&\beta_{k,t}^* = \frac{(1-\rho)p_{k,t}^*P_k S}{w_{k,t}^* W\log\left(1+\frac{P_k h_{k,t}}{w_{k,t}^* W N_0}\right)}, \quad~\forall t\in\mathcal{T},\forall k\in\mathcal{K}\\
&\gamma_k^* = \frac{\rho T^2}{K(\sum_{t=0}^{T-1} p_{k,t}^*)^2}. \quad\quad\quad\quad\quad~~~\forall k\in\mathcal{K}
\end{aligned}
\right.
\end{align}
Then we introduce Lagrangian multipliers $\boldsymbol{\alpha} = \{\boldsymbol{\alpha}_1,\boldsymbol{\alpha}_2,\cdots,\boldsymbol{\alpha}_K\}$ with $\boldsymbol{\alpha}_k = \{\alpha_{k,0}, \alpha_{k,1},\cdots,\alpha_{k,T-1}\}$ and $\boldsymbol{\nu} = \{\nu_1, \nu_2,\cdots,\nu_K\}$ associated with constraints \eqref{eqn:t2proof_st1} and \eqref{eqn:t2proof_st2}, respectively. According to Fritz-John optimality condition, there must exist $(\boldsymbol{\alpha}^*, \boldsymbol{\nu}^*)$ such that
\begin{align}\label{eqn:t2proof3}
\left\{
\begin{aligned}
&\alpha_{k,t}^*=\frac{1}{w_{k,t}^* W\log\left(1+\frac{P_k h_{k,t}}{w_{k,t}^* W N_0}\right)}, ~\quad\forall t\in\mathcal{T},\forall k\in\mathcal{K}\\
& \nu_k^* = 1, \quad\quad\quad\quad\quad\quad\quad\quad\quad\quad\quad\quad \forall k\in\mathcal{K}
\end{aligned} 
\right.
\end{align}
where $(\boldsymbol{p}^*, \boldsymbol{w}^*, \boldsymbol{\beta}^*, \boldsymbol{\gamma}^*)$ is the optimal solution to problem \eqref{eqn:t2proof}. Then, it can be checked that when conditions of $(\boldsymbol{\alpha}^*, \boldsymbol{\beta}^*, \boldsymbol{\gamma}^*)$ in \eqref{eqn:t2proof2} and \eqref{eqn:t2proof3} are satisfied, the optimal $(\boldsymbol{p}^*, \boldsymbol{w}^*)$ of problem \eqref{eqn:t2proof} satisfies the Karush-Kuhn-Tucker (KKT) conditions for problem (P2), where $\boldsymbol{\nu}$ is omitted in problem (P2) since $\nu_k^* = 1,~\forall k\in\mathcal{K}$. As problem (P2) is convex programming for $\boldsymbol{\alpha}, \boldsymbol{\beta}, \boldsymbol{\gamma} \succ 0$, the KKT condition is also sufficient optimality condition. Thus, the optimal $(\boldsymbol{p}^*, \boldsymbol{w}^*)$ of problem \eqref{eqn:t2proof} is also the solution of problem (P2) for given $(\boldsymbol{\alpha}, \boldsymbol{\beta}, \boldsymbol{\gamma}) = (\boldsymbol{\alpha}^*, \boldsymbol{\beta}^*, \boldsymbol{\gamma}^*)$, which completes the proof. 
\subsection{The Details of Obtaining the Close Form Solution $w_{k,t}^*$}
By setting $\frac{\partial L_2}{\partial w_{k,t}} = 0$, we can obtain 
\begin{align}\label{eqn:C_1}
\alpha_{k,t}\beta_{k,t}W\log&\left(1+\frac{P_k h_{k,t}}{w_{k,t}WN_0}\right) \nonumber\\
&- \frac{\alpha_{k,t}\beta_{k,t}WP_k h_{k,t}}{w_{k,t}W N_0 +P_k h_{k,t}} - v_t = 0.
\end{align}
To simplify the notation, we set $\alpha_{k,t}\beta_{k,t}W = a$ and $\frac{P_kh_{k,t}}{WN_0} = b$. Then 
\eqref{eqn:C_1} can be simplified as 
\begin{align}
    a&\left(\log\left(\frac{w_{k,t}+b}{w_{k,t}}\right)-\left(\frac{a+v_t}{a}\right)\right)\exp\bigg(\log\left(\frac{w_{k,t}+b}{w_{k,t}}\right)\nonumber\\
&-\left(\frac{a+v_t}{a}\right)\bigg)\exp\left(\frac{a+v_t}{a}\right) + a = 0.
\end{align}
By setting $\left(\log\left(\frac{w_{k,t}+b}{w_{k,t}}\right)-\left(\frac{a+v_t}{a}\right)\right) = X$, we can obtain
\begin{align}\label{eqn:C_2_5}
    X\exp(X) = -\exp\left(-\frac{a+v_t}{a}\right).
\end{align}
Then the principal branch of the Lambert W function $\mathcal{W}$ is introduced to further simplify \eqref{eqn:C_2_5} as
\begin{align}\label{eqn:C_3}
    X = \mathcal{W}\left(-\exp\left(-\frac{a+v_t}{a}\right)\right).
\end{align}
By Substituting $\alpha_{k,t}\beta_{k,t}W = a$, $\frac{P_kh_{k,t}}{WN_0} = b$ and $\left(\log\left(\frac{w_{k,t}+b}{w_{k,t}}\right)-\left(\frac{a+v_t}{a}\right)\right) = X$ into \eqref{eqn:C_3}, the close form solution of $\widetilde{w}_{k,t}$ can be given by
\begin{align}
    \widetilde{w}_{k,t} = \frac{P_k h_{k,t}}{W N_0 \exp[\mathcal{W}(-\exp(-A_{k,t}))+A_{k,t}]-WN_0},
\end{align}
where $A_{k,t}=\frac{\alpha_{k,t}\beta_{k,t}W+v_t}{\alpha_{k,t}\beta_{k,t}W}$ and considering the constraint $0\leq w_{k,t} \leq 1$, we can obtain the optimal $w_{k,t}^*$ as \eqref{eqn:optimal_w}, which completes the proof.

\bibliographystyle{IEEEtran}
\bibliography{IEEEabrv,link}

\end{document}